\documentclass[letterpaper]{article} % DO NOT CHANGE THIS
\usepackage{aaai25}  % DO NOT CHANGE THIS
\usepackage{times}  % DO NOT CHANGE THIS
\usepackage{helvet}  % DO NOT CHANGE THIS
\usepackage{courier}  % DO NOT CHANGE THIS
\usepackage[hyphens]{url}  % DO NOT CHANGE THIS
\usepackage{graphicx} % DO NOT CHANGE THIS
\usepackage{natbib}  % DO NOT CHANGE THIS AND DO NOT ADD ANY OPTIONS TO IT
\usepackage{caption} % DO NOT CHANGE THIS AND DO NOT ADD ANY OPTIONS TO IT
\usepackage{algorithm}
\usepackage{amssymb}
\usepackage{comment}
\usepackage{amsmath}
\usepackage{bm}
\usepackage{amsfonts}
\usepackage{amsthm}
\usepackage{graphicx}
\usepackage{enumitem}
\usepackage{booktabs}
\usepackage{array}
\usepackage{siunitx}  % For aligning numbers at the decimal point
\usepackage{algpseudocode}

\usepackage{multirow}
\usepackage{subcaption}
\usepackage{bm}
\newtheorem{theorem}{Theorem}[section]

\newtheorem{corollary}[theorem]{Corollary}

\usepackage{pifont}
\usepackage{array}
\newcolumntype{L}[1]{>{\raggedright\let\newline\\\arraybackslash\hspace{0pt}}m{#1}}
\newcolumntype{C}[1]{>{\centering\let\newline  \\\arraybackslash\hspace{0pt}}m{#1}}
\newcolumntype{R}[1]{>{\raggedleft\let\newline \\\arraybackslash\hspace{0pt}}m{#1}}

\DeclareMathOperator*{\argmin}{argmin}

\newtheorem{innercustomthm}{Theorem}
\newenvironment{customthm}[1]
  {\renewcommand\theinnercustomthm{#1}\innercustomthm}
  {\endinnercustomthm}

\usepackage{newfloat}
\usepackage{listings}
\DeclareCaptionStyle{ruled}{labelfont=normalfont,labelsep=colon,strut=off} % DO NOT CHANGE THIS
\floatstyle{ruled}
\newfloat{listing}{tb}{lst}{}
\floatname{listing}{Listing}
\title{BrainMAP: Learning  Multiple Activation Pathways in Brain Networks}

\author {
    Song Wang\textsuperscript{\rm 1}\thanks{Equal contribution.},
    Zhenyu Lei\textsuperscript{\rm 1}\footnotemark[1],
    Zhen Tan\textsuperscript{\rm 2},
    Jiaqi Ding\textsuperscript{\rm 3},
    Xinyu Zhao\textsuperscript{\rm 3},
    Yushun Dong\textsuperscript{\rm 4}, \\
    Guorong Wu\textsuperscript{\rm 3},
    Tianlong Chen\textsuperscript{\rm 3},
    Chen Chen\textsuperscript{\rm 1},
    Aiying Zhang\textsuperscript{\rm 1},
    Jundong Li\textsuperscript{\rm 1}
}
\affiliations {
    \textsuperscript{\rm 1}University of Virginia,
    \textsuperscript{\rm 2}Arizona State University,
    \textsuperscript{\rm 3}University of North Carolina at Chapel Hill,
    \textsuperscript{\rm 4}Florida State University
   \{sw3wv, vjd5zr, zrh6du, xsa5db, jundong\}@virginia.edu, ztan36@asu.edu, 
\{jiaqid, xinyuzh, tianlong\}@cs.unc.edu,
yd24f@fsu.edu, guorong\_wu@med.unc.edu
}
\vspace{-.1in}

\usepackage{bibentry}
\begin{document}

\maketitle

\begin{abstract}
Functional Magnetic Resonance Image (fMRI) is commonly employed to study human brain activity, since it offers insight into the relationship between functional fluctuations and human behavior. To enhance analysis and comprehension of brain activity, Graph Neural Networks (GNNs) have been widely applied to the analysis of functional connectivities (FC) derived from fMRI data, due to their ability to capture the synergistic interactions among brain regions. 
However, in the human brain, performing complex tasks typically involves the activation of certain \textit{pathways}, which could be represented as paths across graphs. As such, conventional GNNs struggle to learn from these pathways due to the long-range dependencies of multiple pathways. 
%However, traditional approaches primarily use pairwise Pearson correlation coefficients of fMRI time series between distinct brain regions to construct the network of functional coactivations, which have limited power to capture high-level neuroscience insights at the circuit level. Additionally, to prevent oversmoothing, many GNNs employ shallow architectures that struggle to capture long-range dependencies across multiple brain regions often involved in the complex nature of cognition and behavior.
To address these challenges, we introduce a novel framework \textbf{BrainMAP} to learn \textbf{\underline{M}}ultiple \textbf{\underline{A}}ctivation \textbf{\underline{P}}athways in \textbf{\underline{Brain}} networks. BrainMAP leverages sequential models to identify long-range correlations among sequentialized brain regions and incorporates an aggregation module based on Mixture of Experts (MoE) to learn from multiple pathways. Our comprehensive experiments highlight BrainMAP's superior performance. Furthermore, our framework enables explanatory analyses of crucial brain regions involved in tasks. Our code is provided at {https://github.com/LzyFischer/BrainMAP}.
\end{abstract}

\section{Introduction}
% 1. teaser brain pathways example figure
% 2. brain network illustration
% 3. model figure
% 4. vidualization

%1. time-series for explanations
%2. moe for potential explanations for multiple paths

%1. analysing the brain important, fmri is useful
%2. gnn is popular in classification.
%3. challenges of gnn: shallow interaction, fixed correlation and connection
%4. recently, graph transformer has gained popularity

%Recently, many research efforts have been devoted to exploring the complicated  patterns of connections in brain functional networks. The advancement of machine learning techniques has promoted research on tasks such as decoding of cognitive processes [14, 15] and diagnosing mental health disorders [16, 17]. Generally, In functional magnetic resonance imaging (fMRI) signals are widely used to measure functional network connectivity [5]. In particular,  fMRI measures blood-oxygen-level-dependent (BOLD) responses, which indicate changes in metabolic demand due to neural activity. By capturing BOLD responses with a distinctive combination of spatial and temporal resolution, fMRI allows for the exploration of complex cognitive processes within the human brain. Furthermore, by studying the functional connective (FC) features that reflect BOLD response, researchers are able to discover the specific behavioral traits and neurological diseases related to particular FC features.

Recently, significant research has focused on learning complex patterns in brain activities, which has promoted tasks such as cognitive process decoding~\cite{li2019interpretable,thomas2022interpreting,finn2023functional} and the diagnosis of mental health disorders~\cite{jo2019deep, eslami2019asd}. 
%
% Functional magnetic resonance imaging (fMRI) data~\cite{fox2007spontaneous, zhang2024metarlec}  measures blood-oxygen-level-dependent (BOLD) responses and reflects changes in metabolic demand associated with neural activity~\cite{kohoutova2020toward, davis2020discovery}.
%
Generally, brain activities could be represented as functional magnetic resonance imaging (fMRI) data~\cite{fox2007spontaneous, zhang2024metarlec}, which measures blood-oxygen-level-dependent (BOLD) responses and reflects changes in metabolic demand associated with neural activity~\cite{kohoutova2020toward, davis2020discovery}. By leveraging fMRI's unique blend of spatial and temporal characteristics, researchers can delve into the complexities of cognitive processes in the human brain~\cite{bassett2017network}. More specifically, %co-activation patterns 
BOLD signals are commonly used to construct networks of brain regions from fMRI data, where the functional connectivities (FC) among distinct brain regions are associated with various normal and pathological states~\cite{kawahara2017brainnetcnn}, as shown in Fig.~\ref{fig:example}. Studying the FC features that correspond to the different brain states
% BOLD responses 
enables the identification of specific behavioral traits and neurological disorders linked to particular FC patterns~\cite{morris2019weisfeiler}.

\begin{figure}[!t]\center
\includegraphics[width=\columnwidth]{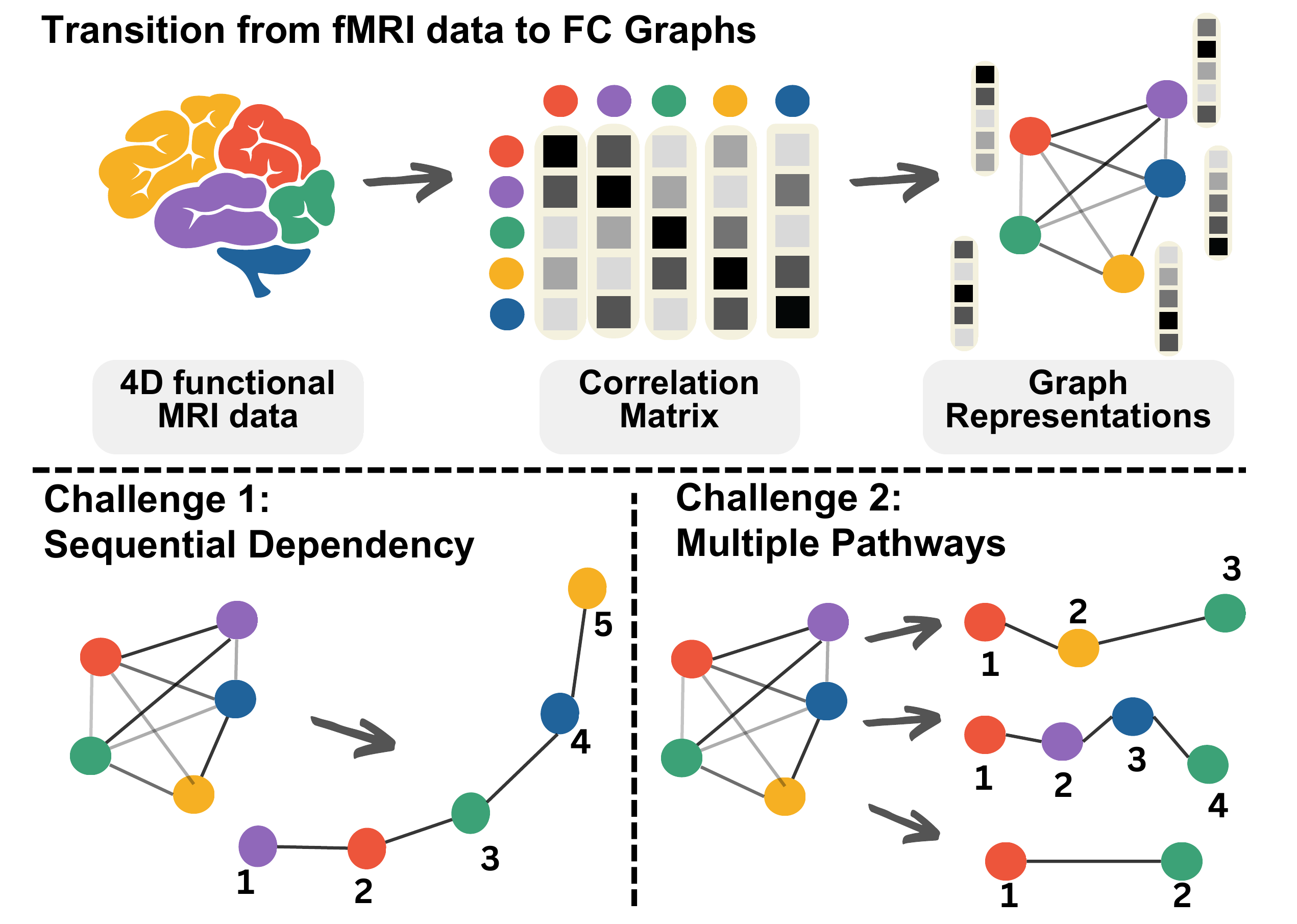}
\vspace*{-.15in}
\caption{An illustration of the transition from fMRI data to FC graphs, along with two challenges for learning from pathways in FC graphs: (1) The sequential dependency, a fundamental feature of human brain activity, is not naturally presented in FC graphs; (2) Multiple pathways exist in FC graphs, making the extraction of them more difficult.}
%\zt{nice figure! please add more descriptions here such that people don't read the intro will still understand the img, also, aaai cannot use vspace}}
\vspace*{-.15in}
\label{fig:example}
\end{figure}

To extract patterns in FC features, they are generally modeled as FC graphs, where nodes represent brain Regions of Interest (ROIs), and edges represent their relationships~\cite{cui2022positional,cui2022braingb}. In this way, the correlations among brain regions could be explicitly represented~\cite{said2023neurograph}. With the development of Graph Machine Learning (GML) techniques, Graph Neural Networks (GNNs) are widely applied to FC graphs~\cite{wang2022contrastive, zhou2020graph}. By capitalizing on the structured nature of the FC graphs and integrating local information, GNNs facilitate learning from patterns in functional connectivities and informative features~\cite{li2021braingnn}.
While FC graphs offer valuable connectivity insights by depicting correlations among brain regions, existing  works often
%struggle to fully exploit the useful knowledge in these graphs~\cite{kim2021learning, ahmedt2021graph}. A major limitation is the tendency to 
overlook the \textbf{\textit{\underline{activation pathways}}} that are inherently present in these graphs.
%as they generally ignore the pathways presented on these graphs.
Specifically, in the human brain, performing tasks typically involves the activation of certain pathways~\cite{sporns2011human}, which could be represented as paths across the FC graphs~\cite{sankar2018dynamic},
%
% each starting from one node (i.e., brain region) to another connected node across multiple edges,
%
as shown in Fig.~\ref{fig:example}. These pathways indicate the transmission     of neural signals to a particular brain region. By incorporating these pathways into analysis, we could capture the complex interactions that might be overlooked when only considering pairwise correlations. Furthermore, these pathways provide insights into how different brain regions segregate into functional modules and collaborate to perform complex tasks.

However, despite the benefits of considering activation pathways on FC graphs, learning from these pathways is challenging. Since they are not explicitly represented in the graphs, without ground truth, models struggle to accurately learn and interpret them. Generally, pathways exhibit two crucial properties as shown in Fig.~\ref{fig:example}: (1) \textbf{Sequential dependency} is a fundamental feature of brain networks, where multiple regions co-activate and interact over long distances~\cite{dahan2021improving}. For example, in emotional memory processing, the hippocampus encodes memories, the amygdala assesses their emotional significance, and the prefrontal cortex uses this information for decision-making, exemplifying the long-range dependencies across multiple brain regions~\cite{said2023neurograph}. Nevertheless, while capturing these sequential dependencies is crucial for understanding the information flow in the brain, the structural nature of FC graphs makes it challenging to effectively model such dependencies.
(2) \textbf{Multiple pathways} are generally necessary for the brain to process different behaviors and perform complex tasks. For example, in visual processing, the brain utilizes two parallel pathways: one along the dorsal visual cortex, which handles fast but coarse information, and the ``what" stream along the ventral visual cortex, which processes slower but more detailed information~\cite{lee2016anatomy}. These distinct pathways correspond to different aspects of visual stimuli, emphasizing the need for multiple pathways in visual processing. Nevertheless, it is especially challenging to capture multiple pathways with existing GNN
%\zt{why suddenly GNN? did you say that GNN is a predominant way before?} 
architectures. Due to the inherent limitations of the message-passing mechanism~\cite{kipf2017semi, velivckovic2017graph, tan2022transductive}, which focuses on aggregating information from neighboring nodes, GNNs struggle to effectively model the complex, long-range interactions in multiple pathways~\cite{kim2021learning}. Moreover, the interpretability of functional connectivity patterns is underexplored in current GNN-based approaches. Existing interpretable GNN models~\cite{ying2019gnnexplainer, luo2020parameterized}, which are typically designed to explain the importance of individual nodes and edges rather than considering their relationships within an activation path, struggle to provide explanations for interactions in long-range paths.

%where one processes coarse but fast information along the dorsal visual cortex, while the "what" stream processes fine but slow information along the ventral visual cortex. These segregated pathways correspond to different aspects of visual stimuli, highlighting the need for multiple pathways in visual processing. Nevertheless, despite the importance of pathways in brain activities, they are difficult to capture with existing GNN architectures due to the message-passing mechanism.

In this work, we propose \textbf{BrainMAP} to effectively learn from and interpret \textbf{\underline{M}}ultiple \textbf{\underline{A}}ctivation \textbf{\underline{P}}athways present in FC (functional connectivities) graphs while tackling the challenges posed by long-range dependencies and pathway correlations. To achieve this: (1) We propose an \textbf{Adaptive Graph Sequentialization} module to transform each FC graph into a node sequence that reflects the order of information flow, which enables the extraction of the hidden pathways that are crucial for modeling long-range interactions. 
%By in order to extract the potential pathways and process them with sequential models (such as Mamba)~\cite{gu2023mamba}
%, which is particularly effective in modeling long sequential data~\cite{gu2023mamba}.
%Mamba architecture, which has gained prominence in modeling long sequential data, to effectively handle long-range dependencies within pathways on  FC graphs.  This design enables our model to capture distant interactions between brain regions along the pathways, which are essential for understanding intricate neural processes. 
(2) We design a \textbf{Hierarchical Pathway Integration} strategy that analyzes correlations among multiple pathways. Inspired by the human brain's use of parallel pathways in complex tasks, we propose to integrate insights from diverse pathways, which captures complementary information contributed by each pathway. 
More importantly, our design improves interpretability by identifying the crucial brain regions in pathways that work together to support brain functions.
% In consequence, our framework could also provide explanations that identify the crucial brain regions responsible involved in specific pathways for various tasks. 
Such interpretability offers deeper insights into the functional %organization 
co-activated pattern of the brain. To evaluate our framework, We conduct experiments on five real-world fMRI datasets. The results demonstrate that our framework outperforms existing models in various prediction tasks on FC graphs while also offering comprehensive explanations for pathways. 
In summary, our contributions are as follows:
\begin{itemize}[leftmargin=0.4cm]
    \item \underline{\textbf{\textit{Innovation}.}} We present a novel framework for predictive tasks on FC graphs while providing comprehensive explanations to identify crucial brain regions—an area that has been underexplored in prior research.
    \item \underline{\textbf{\textit{Architecture}.}} We design an Adaptive Sequentialization module to transform FC graphs into node sequences for pathway learning, and a Pathway Integration module to aggregate and analyze correlations across multiple pathways on FC graphs.
    %\item \underline{\textbf{\textit{Interpretation}.}} Our framework also offers interpretability by identifying and explaining the key brain regions responsible for specific tasks. This ability to trace pathways and understand the underlying neural processes provides valuable insights for neuroscience research.
    \item \underline{\textbf{\textit{Validation}.}} We conduct extensive experiments on various real-world FC datasets, and the results demonstrate the superior performance of our framework in both predictions and explanations.
    %\zt{how can you justify/evaluate the explanantion? by some metrics or by human domain experts like Jiaqi? mention it here}comprehensiveness.
   % We conduct experiments on various real-world datasets and the results demonstrate superior precision performance and explanation comprehensiveness. 
\end{itemize}

\section{Related Work}
%\zt{i suggest put the related work at the former place. otherwise people won't understand}
\subsection{Brain Network Analysis}
% what is the task, predict, is there any other tasks? what previous work do to solve those tasks? what are their limitations? existing works use GNN for brain networks
Brain network analysis aims to understand the intricate patterns of connectivity within the brain~\cite{cui2022braingb, kan2022brain, zhang2022probing,hsu2024thought, zhang2024metarlec, gao2024local}, which has gained increasing popularity recently due to its various applications, including identifying biomarkers for neurological diseases~\cite{chang2021machine, yang2022data}, understanding cognitive processes~\cite{liu2023coupling, article}, and distinguishing different types of brain networks~\cite{liao2024joint}. Among these, one of the most important tasks is the prediction of brain-related attributes, such as demographics and task states~\cite{said2023neurograph, he2020deep}. Recently, GNNs have significantly evolved as a major field of exploration for these tasks~\cite{li2022brain, cui2022braingb}, due to their extraordinary ability to leverage the structured data~\cite{li2021braingnn,xu2024learning, wang2022glitter}. 
% In addition, recent works have introduced specialized GNNs tailored for certain tasks to gain better prediction results~\cite{}. 
Nevertheless, GNN-based approaches often struggle to fully exploit the useful knowledge in brain networks, particularly the activation pathways that are inherently present in brains~\cite{keller2018task} are neglected. To address this limitation, we propose to extract multiple underlying activation pathways with adaptive structure sequentialization and Mixture of Experts (MoE), which thus enables a more comprehensive understanding of the brain connection. 

%\subsection{State Space Models}
% what is state space model, what is mamba, why it's in graph, and how, what is the limitation
%State space models (SSMs) are .... Recently, a SSM-based model ... (Mamba) has shown superior performance in various domains including ... Although it's designed for structured sequentialized inputs (time series, language), several recent works have made attempts to apply Mamba to non-Euclidean graph data. However, their sequentialization is not ideal. In this work, we ...

\subsection{Mixture of Experts}
% why moe? widely used in other domains. moe on mamba, how did they do? however, they can't satisfy other requirements.
%Mixture of experts have been widely adopted in a variety of domains including. ... has applied MoE to Mamba, where .... However, each expert in ... is MLP layer, which fail to fully unleash the power of utilizing multiple Mamba. In our work, ...

The Mixture of Experts (MoE) approach involves deploying a collection of expert networks, each designed to specialize in a particular task or a subset of the input space~\cite{shazeer2017outrageously,wang2024one}. Originally derived from traditional machine learning models~\cite{jacobs1991adaptive,jordan1994hierarchical}, MoE has since been adapted for deep learning, significantly enhancing its ability to handle complex vision and language tasks~\cite{jiang2024mixtral}. 
In addition to the strategy of interesting  MoE layers with conventional neural networks~\cite{vaswani2017attention, dauphin2017language}, the concept of MoE is also extended to large and independent modules, e.g., language models as agents~\cite{wang2024mixture}.
In this work, we extend the MoE framework to address the challenge of multiple pathways in brain networks, focusing on learning the correlations across pathways. As a result, our framework is able to extract multiple pathways within and across different orders while learning from them.

%The idea behind Mixture of Experts (MoE) is to have a set of expert networks, each specializing in a particular task or a subset of the input space~\cite{shazeer2017outrageously,vaswani2017attention,jiang2024mixtral}.
% The Mixture of Experts (MoE) originated from studies in traditional machine learning models~\cite{jacobs1991adaptive} and was subsequently adapted to deep learning to bolster its capacity to address intricate vision and language-related challenges~\cite{shazeer2017outrageously,vaswani2017attention,jiang2024mixtral}. 
%Wang et al. extended this paradigm to the prompt optimization task, achieving substantial performance improvements~\cite{wang2023mixture}. However, their approach overlooks the potential benefits of leveraging multiple expert collaborations. We extend the MoE framework to tackle the demonstration selection problem, aiming to effectively navigate the demonstration pool while considering the interplay among in-context examples.

\section{Preliminary}

% To construct brain graphs, a conventional methodology is to create each FC graph from an fMRI scan~\cite{said2023neurograph}.
%
In this work, we define an FC graph \( G \) as \( G = (\mathcal{V}, \mathcal{E}, \mathbf{A}, \mathbf{X}) \), where \( \mathcal{V} \) represents the set of nodes that indicate brain regions, \( \mathcal{E}\) denotes the edges that illustrate functional connections between these regions, \( \mathbf{A}\in\mathbb{R}^{N\times N} \) is the adjacency matrix capturing the connectivity structure, and \( \mathbf{X}\in\mathbb{R}^{N\times d} \) denotes node features that may include various biological markers or other relevant attributes. The total number of vertices in the graph is represented by \( N \), such that \( |\mathcal{V}| = N \), and let $d$ be the number of dimensions in the input feature of each node. We use ${Y}$ to denote the prediction target of each graph in classification or regression tasks.
% \sw{maybe we add labels or regression value to each graph?}

%\subsection{State Space Models}

\begin{figure*}[!t]\center

\includegraphics[width=0.95\textwidth]{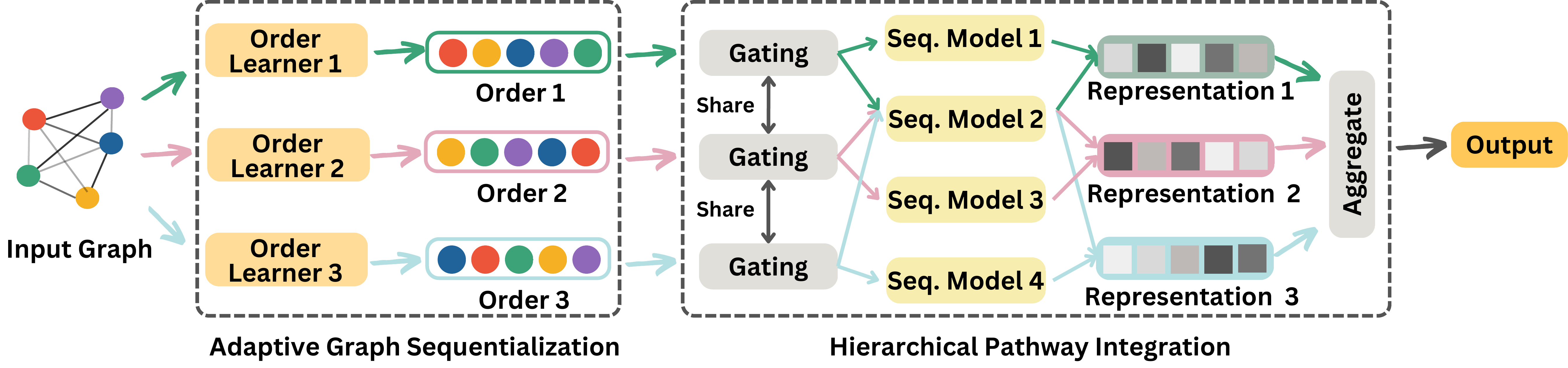}
\vspace*{-.05in}
\caption{The overall process of BrainMAP. We first adaptively learn $M$ ($M=3$ in the figure) orders with three order-learners (GNNs). Then these orders are input into the gating function to select $K$ ($K=2$ in the figure) experts from a total number of $P$ ($P=4$ in the figure) experts. Each expert is implemented as a sequential model. The output of these experts will be aggregated into a representation for each order. Finally, the representations from all orders are aggregated into the output.}
\vspace*{-.1in}
\label{fig:framework}
\end{figure*}

\section{Methodology}
An overview of BrainMAP is presented in Fig~\ref{fig:framework}. Specifically, BrainMAP is composed of two components: (1) \textit{\textbf{Adaptive Graph Sequentialization}}, which learns the optimal sequence of brain regions by transforming the FC graph structure into a meaningful order that captures key dependencies, 
%which leverages Mamba's selective scanning capability to extract critical long-range information flow paths;  
    and (2) \textit{\textbf{Hierarchical Pathways Aggregation}}, which utilizes multiple experts to extract diverse pathways from different orders and then aggregates them to capture complex interactions across multiple pathways. Each expert is instantiated as a sequential model such as Transformer~\cite{vaswani2017attention} or Mamba~\cite{gu2023mamba}, in order to extract long-range dependencies within potential pathways. %multiple potential pathways while enhancing diversity through a mixture of experts. \sw{introduce experts }

\subsection{Adaptive Graph Sequentialization}
\label{4.1}
% However, it's difficult to explicitly leverage the sequential activation knowledge since it's an non-trivial task to monitor the sequential activation of brain regions in a live body.

When performing tasks such as visual or motor activities, research has shown that multiple brain regions often collaborate over long distances rather than functioning in isolation, which means cognitive processes emerge from the sequential activation of these regions~\cite{thiebaut2022emergent}. Consequently, capturing the order of sequential activation paths is crucial for accurate prediction in brain networks. Nevertheless, due to the complex structure of FC graphs, it is difficult to identify and extract such pathways. To address this, we propose an adaptive sequentialization strategy that transforms each FC graph into a sequence, in order to preserve key pathway information and facilitate more effective modeling of the brain's dynamic processes.

%Drawing inspiration from recent advancements in state space models,  we propose to utilize Mamba to achieve this goal by implicitly extracting long-range activation paths in the brain with the selective scanning and long-range processing capabilities of Mamba.

 %However, this approach is challenging because state space models are originally designed for sequential data rather than FC graphs, and they can not directly extract activation pathways from the graph structure. To overcome this challenge, we propose a carefully designed pipeline for selective path extraction in FC graphs using Mamba. This pipeline consists of two key steps: (1) \textit{\textbf{Learning-based Structure Sequentialization}}, where we flatten the FC graphs into node sequences, with the order of nodes learned by GNNs optimized for capturing critical structures; and (2) \textit{\textbf{Multiple Pathways Aggregation}}, which identifies and selects crucial brain regions within the flattened node sequence to build a sequential activation path, thereby providing implicit knowledge for downstream brain-related tasks.

\paragraph{Learning Orders for FC Graphs.} 
% Since the sequential nature of selective state space models require the input to have linear structure, Mamba can't be directly be applied to brain networks, which induces us to flatten the brain networks as the first step. However, a 
A significant obstacle in converting FC graphs into node sequences lies in the permutation invariance of brain network regions (i.e., nodes)~\cite{said2023neurograph}. This invariance contrasts with the inherently sequential nature of activation pathways, which do not naturally account for such invariance. % As a result, when the sequential order of brain regions in the input is random or fails to capture the underlying structure of the brain, the output of state space models can lead to inaccurate predictions.
%inaccurate prior assumptions when the sequential order of brain regions in the input is random or fails to capture intrinsic information about the brain's structure.

To tackle this, we introduce a learning-based strategy that utilizes an order-learning GNN to adaptively determine the node order for each input FC graph. 
%Particularly, we propose to learn an ordering score for each node in the input graph with an order- GNN. 
With the order-learning GNN, we aim to learn the optimal sequence of nodes by arranging them based on their learned ordering scores in an ascending order. The benefit of using the learned scores to describe the order is that it avoids the massive search space of possible node orders (i.e., $N!$ for a graph of size $N$), which would otherwise make exhaustive search infeasible. %Note that the GNN is specific for each expert, and the optimization is performed individually along with each expert.
In the following, we describe the process of learning the ordering scores with order-learning GNN. Given an input graph $G$, the ordering score  $s_i\in\mathbb{R}$ of node $v_i$ in $G$ is learned as:
\begin{equation}
    s_i=\text{GNN}_l(\mathcal{V}_i, \mathcal{E}_i, \mathbf{X}_i),\ \ \text{where}\ \ \mathcal{V}_i=\mathcal{N}_i\cup\{v_i\}.
    \label{eq:ordering_score}
\end{equation}
Here $\mathbf{X}_i$ is the feature matrix of $\mathcal{V}_i$, which is the set of neighboring nodes of $v_i$. $\mathcal{E}_i$ is the set of edges for nodes in $\mathcal{V}_i$. $\text{GNN}_l$ is the order-learning GNN. % for the $k$-th expert. (including $v_i$ itself)
%

%and then force the learned scores of each node in ascending order to approach the orders with larger output probabilities. 
 With the ordering scores $\{s_1, s_2,\dotsc, s_N\}$ of nodes in $G$, calculated in Eq.~(\ref{eq:ordering_score}), we obtain the order $\hat{\boldsymbol{\phi}}$ of $N$ nodes $\{v_1, v_2,\dotsc, v_N\}$ as follows:
\begin{equation}
\begin{aligned}
    \hat{\boldsymbol{\phi}}&=(v_{\pi(1)},v_{\pi(2)},\dotsc, v_{\pi(N)}), \\ \text{where}&\ \pi(i) = \argmin_{j \notin\{\pi(1), \pi(2),\dots, \pi(i-1)\}} s_j.
\end{aligned}
\end{equation}
Here $\pi$ is a permutation of indices that sorts the scores $\{s_1,s_2,\dotsc, s_N\}$ in an ascending order. $\hat{\boldsymbol{\phi}}$ denotes the obtained order of $N$ nodes in the input graph.

\paragraph{Optimization of Order-Learning GNNs.} 
To optimize the order-learning GNN, an obvious challenge is the lack of ground-truth orders. That being said, the optimal node order that consists of sufficient pathway information remains unavailable. Therefore, we propose to use the loss of BrainMAP output to select good and bad orders as the supervision signal. Intuitively, the orders that could provide smaller losses regarding the correct label (or ground-truth values in regression tasks) should be more similar to the optimal orders.
%In other words, we assume that the orders leading to larger correct output probability are more similar to the optimal orders. 
%
In concrete, within each training step, we first randomly sample a batch of orders and compute their corresponding output. Then we select $N_p$ orders with the smallest losses as positive samples (denoted as $\Phi_p$), and select $N_d$ orders with the largest losses as negative samples (denoted as $\Phi_n$). Based on the concept of contrastive learning~\cite{you2020graph,tan2022supervised,xu2023cldg,xu2024learning,wang2023contrast}, our optimization aims to increase the similarity between the learned order and the positive orders, while decreasing the similarity between the learned order and the negative orders. In this manner, we manage to gradually make the learned order approach the better orders during training. 
\begin{equation}
    \max \sum\limits_{\boldsymbol{\phi}\in{\Phi}_p} s(\hat{\boldsymbol{\phi}},{\boldsymbol{\phi}}) -\lambda\sum\limits_{\boldsymbol{\phi}\in{\Phi}_n} s(\hat{\boldsymbol{\phi}},{\boldsymbol{\phi}}) ,
\end{equation}
where $\lambda$ determines the relative importance of the two objectives. $s(\hat{\boldsymbol{\phi}},{\boldsymbol{\phi}})$ denotes the similarity between $\hat{\boldsymbol{\phi}}$ and ${\boldsymbol{\phi}}$.

To estimate the similarity between any two orders, a straightforward strategy is to leverage Spearman's rank correlation coefficient~\cite{spearman1904proof}. However, this coefficient is computed between two real ranks, each consisting of distinct values from 1 to $N$, which contrast the ordering scores we obtain. Moreover, directly converting the ordering scores into real ranks (i.e., from $1$ to $N$) would prevent the flow of gradients,  making optimization through gradient descent infeasible. To deal with this issue, we propose to calculate approximate rank scores as a substitute for the original ordering scores (i.e., $ s_i$). These rank scores are differentiable and can be optimized using gradient descent.
%and propose to increase or decrease the coefficients between pairs of orders for optimization. However, 

Denoting any learned order $\hat{\boldsymbol{\phi}}$, along with its ordering scores $\{s_1, s_2,\dotsc,s_N\}$,
%, we compute the approximate rank score $S_i$ for each node $v_i$ as follows. 
%This is because compared to directly ranking each ordering score as a specific integer, using a soft score enables the backward propagation of gradients for optimizing the learner GNN. In particular, 
we calculate the approximate rank score $S_i\in\mathbb{R}$ of node $v_i$ as follows:
\begin{comment}
    \begin{equation}
    S_i=\frac{\sum\nolimits_{j=1}^N \sigma(s_i-s_j)-\sigma(s_{min}-s_j)}{\sum\nolimits_{j=1}^N \sigma(s_{max}-s_j)-\sigma(s_{min}-s_j)}\cdot (N-1)+1,
\end{equation}
where $\sigma(s)=1/(1+\exp(-s))$. Moreover, $s_{max}=\max\{s_1,s_2,\dotsc, s_N\}$ and $s_{min}=\min\{s_1,s_2,\dotsc, s_N\}$. Notably, in the calculation, we consider the difference between any two ordering scores in order to take into account the relative position of each score within the entire set of scores. 
\end{comment}
\begin{equation}
\begin{aligned}
       & S_i=\frac{s_i - \mathbb{E}[s] }{\sqrt{\mathbb{E}[s^2] - (\mathbb{E}[s])^2}}\cdot\sqrt{\frac{N^2-1}{12}}+\frac{(N+1)(2N+1)}{6}, \\
       &\quad\text{where}\quad  \mathbb{E}[s]=\frac{1}{N}\sum\limits_{i=1}^Ns_i, \quad \mathbb{E}[s^2]=\frac{1}{N}\sum\limits_{i=1}^Ns_i^2. \\
\end{aligned}
\label{eq:soft_rank}
\end{equation}
In Eq.~(\ref{eq:soft_rank}), the ordering score $s_i$ is linearly transformed to $S_i$, based on the mean and variance of all $N$ ordering scores, i.e., $\{s_1, s_2,\dotsc, s_N\}$. We perform such transformation to ensure that the mean and variance of $\{S_1, S_2,\dotsc, S_N\}$ are the same as those of a real rank variable of size $N$, which results in a similar distribution. Moreover, it also aligns with our loss design of enhancing the Spearman’s rank correlation coefficient~\cite{spearman1904proof}, as introduced later.
The consistency of mean and variance is verified in the following theorem.
\begin{theorem}
    The mean and standard deviation of $S_i$ are the same as those of any real rank variable $R$ for a sample size of $N$, i.e., 
    \begin{equation}
        \mu(S_i)=\mu(R), \quad \sigma(S_i)=\sigma(R).
    \end{equation}
    \label{theorem1}
\end{theorem}
\noindent We provide the proof in Appendix~\ref{app:theorem1}. %
According to Theorem~\ref{theorem1}, rank scores $\{S_1, S_2,\dotsc, S_N\}$ could represent an approximate rank as they share the same mean and variance.

\begin{table*}[!t]
\setlength\tabcolsep{6.8pt}%调列距
\centering	
\renewcommand{\arraystretch}{1.0}

\vspace{0.05in}
\resizebox{0.9\textwidth}{!}{
\begin{tabular}{lcccccccccc}
\toprule[1pt]
\textbf{Dataset} & \textbf{$|G|$} & \textbf{$|N|_\text{avg}$} & \textbf{$|E|_\text{avg}$} & \textbf{$d_\text{max}$} & \textbf{$d_\text{avg}$} & \textbf{$K_\text{avg}$} & $d_x$ & \textbf{\#Classes} & \textbf{Prediction Task} \\
\midrule
HCP-Task & 7,443 & 360 & 7,029.18 & 153 & 17.572 & 0.410 & 360 & 7 & Graph Classification \\
HCP-Gender & 1,078 & 1,000 & 45,578.61 & 413 & 45.579 & 0.466 & 1,000 & 2 & Graph Classification \\
HCP-Age & 1,065 & 1,000 & 45,588.40 & 413 & 45.588 & 0.466 & 1,000 & 3 & Graph Classification \\
HCP-FI & 1,071 & 1,000 & 45,573.67 & 413 & 45.574 & 0.466 & 1,000 & - & Graph Regression \\
HCP-WM & 1,078 & 1,000 & 45,578.61 & 413 & 45.579 & 0.466 & 1,000 & - & Graph Regression \\
\bottomrule[1pt]
\end{tabular}
}
\caption{The detailed statistics of datasets used in our experiments. $|G|$ denotes the number of graphs in each dataset, $|N|_\text{avg}$ and $|E|_\text{avg}$ represent the average number of nodes and edges, respectively. $d$ signifies the degree, and $K_\text{avg}$ represents the global clustering coefficient. The datasets encompass two types of prediction tasks: graph classification and graph regression.}
\label{tab:dataset_statistics}
\end{table*}

\noindent\textbf{Optimization Loss.} For any randomly sampled (real) order ${\boldsymbol{\phi}}$, 
we use $S_i^\phi$ to represent the rank of $v_i$ in $\boldsymbol{\phi}$. Since ${\boldsymbol{\phi}}$ is a real order, we know $S_i^\phi$ is an integer and $1\leq S_i^\phi \leq N$. Moreover, $S^\phi_i \neq S^\phi_j$ if $i\neq j$. Given the approximate rank scores $\{S_1, S_2,\dotsc, S_N\}$ of a learned order $\hat{\boldsymbol{\phi}}$, we optimize it according to the following loss:
%(because brains perform actions in multiple pathways), 
%we denote the rank of $v_i$ at ${S}^\phi_i$, which means $v_i$ is at the ${S}^\phi_i$-th position in $\boldsymbol{\phi}$.
\begin{equation}
\begin{aligned}
 \hspace{-.05in}   \mathcal{L}(\hat{\phi},{\Phi}_p, {\Phi}_n)
=\frac{\sum\limits_{\boldsymbol{\phi}\in{\Phi}_p}\sum\limits_{i=1}^N (S_i- {S}^\phi_i)^2}{\sum\limits_{\boldsymbol{\phi}\in{\Phi}_p}\sum\limits_{i=1}^N (S_i- {S}^\phi_i)^2+ \sum\limits_{\boldsymbol{\phi}\in{\Phi}_n}\sum\limits_{i=1}^N (S_i- {S}^\phi_i)^2},
\end{aligned}
    \label{eq:loss}
\end{equation}
Here $\Phi_n$ and $\Phi_n$ are the set of sampled positive orders and negative orders, respectively. $|\Phi_p|=N_p$ and $|\Phi_n|=N_n$.
To validate the effectiveness of using loss $\mathcal{L}(\hat{\phi},{\Phi}_p, {\Phi}_n)$ for optimization, we propose the following theorem.
\begin{theorem}
    Minimizing the loss $\mathcal{L}$ described in Eq.~(\ref{eq:loss}) equals maximizing the Spearman's rank correlation coefficient~\cite{spearman1904proof} between learned orders and good orders, while minimizing the coefficient between learned orders and bad orders.
    \begin{equation}
       \min \mathcal{L} (\hat{\phi},{\Phi}_p, {\Phi}_n)  \equiv \max \frac{\sum\limits_{\boldsymbol{\phi}\in{\Phi}_n}\left(1-r(\hat{\phi},\phi)\right)}{\sum\limits_{\boldsymbol{\phi}\in{\Phi}_p}\left(1-r(\hat{\phi},\phi)\right)},
    \end{equation}
    where the coefficient is calculated as:
    \begin{equation}
        r(\hat{\phi},\phi)=\rho_{S,S^\phi}=\frac{\text{cov}(S, S^\phi)}{\sigma_S \sigma_{S^\phi}}.
    \end{equation}
    \label{theorem2}
\end{theorem}
\noindent The proof of Theorem~\ref{theorem2} is provided in Appendix~\ref{app:theorem2}. According to Theorem~\ref{theorem2}, we know that optimizing the loss $\mathcal{L}$ as described in Eq.~(\ref{eq:loss}) can increase the Spearman's rank correlation coefficient between learned orders and good orders. Moreover, the objective also decreases the coefficient between learned orders and bad orders. In concrete, with the theoretical support from Theorem~\ref{theorem2}, we manage to optimize the order-learner with sampled good and bad orders. % 

\subsection{Hierarchical Pathway Integration}
% existing problem:
% 1. different pathways are from different orders, 2. but here we use MoE to extract multiple pathways 
% different sequentialization must use different models

Although sequential models can extract long-range pathways, they are inherently limited to identifying a single pathway at a time. In contrast, the human brain typically relies on multiple pathways to process various behaviors and perform complex tasks, as different pathways often contribute unique and complementary information~\cite{morris2019weisfeiler}. 
%\sw{We need an example like this For example, in visual tasks, information flows from the retina to the thalamus, then to the primary visual cortex (V1) for initial processing, and subsequently to higher-level areas for further analysis~\cite{}.}
For instance, in visual processing, the brain employs two parallel pathways: one along the dorsal visual cortex, which quickly processes broad, less detailed information, and another one along the ventral visual cortex, which handles slower but more detailed information~\cite{lee2016anatomy}.

To deal with the challenge of multiple pathways, we propose to learn numerous activation pathways from each order of brain regions with multiple sequential models. Moreover, the activation pathways can be present in different orders. To effectively learn from these diverse pathways, we propose a two-level hierarchical integration approach, across and within different orders. (1) We first utilize the Mixture of Experts (MoE) strategy to integrate multiple pathways within each sequential order of brain regions. (2) Next, we aggregate the representations across different orders to obtain a comprehensive representation of brain activity.

%We firstly utilize the Mixture of Experts (MoE) for the integration of multiple pathways within each order of brain regions. The aggregated representations from different orders are then further aggregated to get a comprehensive representation for the brain.  This approach aims to capture the diverse and complex interactions necessary for accurately modeling intricate brain functions. 

\paragraph{$\blacktriangleright$ Step 1: Pathway Aggregation within Each Order.}

% Input, each order for one expert? or each order choose from different experts? 
% what is multiple pathways? multiple means multiple sequentialization or experts?
% What is multiple 

Within each order, multiple sub-sequences may connect different sets of brain regions that appear as activation pathways, while they can be hard to extract with only one sequential model, due to the potential heterogeneity among pathways. Thus, we propose to learn multiple pathways simultaneously based on the MoE architecture, with each expert capturing different underlying pathways. To be specific, BrainMAP consists of multiple experts, each utilizing a different sequential model. To dynamically determine which experts are most suitable for a specific order, we design a gating function that ensures the similar pathways are consistently assigned to the same expert. In this manner, each expert specializes in capturing a specific type of pathway. 

Formally, considering an input order $\boldsymbol{\hat{\phi}}$ and $P$ experts, the aggregation is performed as follows:
\begin{equation}
\boldsymbol{z}^{\prime}=\sigma\left(\sum_{i=1}^P G_i(\boldsymbol{\hat{\phi}}) F_i(\boldsymbol{\hat{\phi}})\right) \text {, }
\end{equation}
where $F_i$ is the sequential model of the $i$-th expert. $G$ is the gating function that generates multiple decision scores with the input as sequentialized brain regions $\boldsymbol{\hat{\phi}}$, and $G(\boldsymbol{\hat{\phi}}) \in \mathbb{R}^P$ denotes the scores to choose $P$ experts for the graph. 
We employ an attention-based top-k gating design for $G$, which can be formalized with
\begin{equation}
    \begin{aligned}
& G(\boldsymbol{\hat{\phi}})=\operatorname{Softmax}(\operatorname{TopK}(Q(\boldsymbol{\hat{\phi}}), K)), \\
& Q(\boldsymbol{\hat{\phi}})=\operatorname{MLP}\left(\operatorname{Attention}\left(\mathbf{Q}, \mathbf{K}, \mathbf{V}\right)\right),\\
&
\mathbf{Q} = \mathbf{W}_Q \boldsymbol{h}, \quad \mathbf{K} = \mathbf{W}_K \boldsymbol{h}, \quad \mathbf{V} = \mathbf{W}_V \boldsymbol{h}, \\
&
\boldsymbol{h} = \mathbf{W}_I\boldsymbol{\hat{\phi}} + \operatorname{PE}(\boldsymbol{\hat{\phi}}),
\end{aligned}
\end{equation}
where $\operatorname{PE}$ denotes the sinusoidal positional encoding, which is utilized to inform the gating function with the order information of sequentialized graph representations. $\mathbf{W}_I$, $\mathbf{W}_Q$, $ \mathbf{W}_K$, and $\mathbf{W}_V$ are learnable parameters, and $\operatorname{Attention}$ denotes self-attention mechanism. Besides, $K$ denotes the number of selected experts ($K\leq P$). 
$\text{TopK}(Q(\boldsymbol{\hat{\phi}}),K)$ denotes that we keep the top $K$ values in $Q(\boldsymbol{\hat{\phi}})$, i.e.,
\begin{equation}
\begin{aligned}
        &\operatorname{TopK}(Q(\boldsymbol{\hat{\phi}}), K)_j\\
        &= \begin{cases}Q(\boldsymbol{\hat{\phi}})_j & \text { if } Q(\boldsymbol{\hat{\phi}})_j \text { is in  the top } K \text { values of } Q(\boldsymbol{\hat{\phi}}), \\ -\infty & \text { otherwise. }\end{cases}
\end{aligned}
\end{equation}
%
% To further diversify the pathways learned by each Mamba expert, we make the input brain sequence different for different experts, where $\boldsymbol{\Phi}^i \neq \boldsymbol{\Phi}^j, \ \ \operatorname{if} i \neq j$. By forwarding brain sequences with different orders to different Mambas, it improves the diversity of Mambas by encouraging each Mamba to select different brian regions from distinct views inherent in diverse brain sequence orders. As a result, each expert of Mamba learns to focus on different pathway extraction mechanisms. 

\paragraph{$\blacktriangleright$ Step 2: Pathway Aggregation across Different Orders.}
After the aggregation in Step 1, we obtain an output representation from each order.
To aggregate the pathway information across different orders, we compute the weighted sum over representations learned from these orders, and the weights are the maximum value of $Q(\boldsymbol{\hat{\phi}})$. In this manner, we achieve a final embedding for the input FC graph, i.e., 
\begin{equation}
    \boldsymbol{z} = \sum_{i=1}^M \operatorname{Max}(Q(\boldsymbol{\hat{\phi_i}})) \cdot \boldsymbol{z'}_i
    % where \ y'_i = \sigma\left(\sum_{k=1}^K G(X_k)\circ\operatorname{Mamba}^k(\boldsymbol{\hat{\phi_i}})\right)
\end{equation}
where $\boldsymbol{z'}_i$ is the representation learned from the $i$-th order. 

For the training of gating functions and experts in BrainMAP, we adopt the cross-entropy (CE) loss for classification and the mean absolute error (MAE) for regression tasks.

\section{Experiments}

In this section, we aim to answer the following research questions (RQs). \textbf{RQ1.} How well can BrainMAP perform on brain-related tasks compared to other alternatives? \textbf{RQ2.} How does each component contribute to the overall predictive performance? \textbf{RQ3.} How effectively can BrainMAP elucidate the rationale behind its predictive outcomes?
\textbf{RQ4.} What impact does the design of MoE have on performance?

\begin{table*}[]
\centering
		\setlength\tabcolsep{9pt}%调列距
\renewcommand{\arraystretch}{1}
\begin{tabular}{lcccccc}
\toprule
\textbf{Dataset} & {HCP-Task$\uparrow$} & {HCP-Gender$\uparrow$} & {HCP-Age$\uparrow$} & {HCP-FI$\downarrow$} & {HCP-WM}$\downarrow$ \\
%\textbf{Evaluation Metric  }             & {Accuracy $\uparrow$} & {Accuracy $\uparrow$} & {Accuracy $\uparrow$} & {MAE $\downarrow$} & {MAE $\downarrow$} \\
\midrule
GCN            & 86.29~$(\pm 0.98)$	& 76.03~$(\pm 2.40)$	& 44.27~$(\pm 2.69)$	& 11.49~$(\pm 0.15)$	& 3.95~$(\pm 0.05)$  \\
GAT            & 85.60~$(\pm 1.26)$	& 75.62~$(\pm 2.22)$	& 44.48~$(\pm 2.35)$	& 13.69~$(\pm 0.52)$	& 4.06~$(\pm 0.11)$ \\
SAGE       & 84.49~$(\pm 0.57)$	& 74.69~$(\pm 3.50)$	& 45.83~$(\pm 1.78)$	& \underline{11.34}~$(\pm 0.12)$	& 3.99~$(\pm 0.06)$ \\
ResGCN & 93.75~$(\pm 0.35)$	& 76.75~$(\pm 0.65)$	& 43.54~$(\pm 0.90)$	& 11.48~$(\pm 0.29)$	& \underline{3.92}~$(\pm 0.04)$ \\
GraphGPS   & 92.13~$(\pm 2.00)$	& 76.85~$(\pm 1.54)$	& 45.84~$(\pm 3.21)$	& 11.37~$(\pm 0.86)$	& 3.98~$(\pm 0.04)$ \\
Graph-Mamba     & \underline{94.17}~$(\pm 0.86)$	& \underline{77.16}~$(\pm 3.13)$	& \underline{46.35}~$(\pm 2.73)$	& 11.51~$(\pm 0.88)$	& 3.94~$(\pm 0.14)$ \\
\midrule
\textbf{BrainMAP}     & \textbf{94.74}~$(\pm 0.07)$	& \textbf{78.92}~$(\pm 0.49)$	& \textbf{48.44}~$(\pm 1.65)$	& \textbf{10.75}~$(\pm 0.61)$	& \textbf{3.81}~$(\pm 0.03)$ \\
\bottomrule
\end{tabular}
\caption{Performance comparison of different models across various datasets. The best performance and the second-best performance are in bold and underlined, respectively. All experiments are repeated with 3 different random seeds. } %It is observed that BrainMAP consistently outperforms baselines on all HCP datasets, which demonstrates its effectiveness on brain-related tasks.
\label{tab:main}
\end{table*}

\subsection{Experimental Settings}
We provide a brief introduction to the experimental settings. For the sequential model in our framework, we utilize the Mamba~\cite{gu2023mamba}, which is particularly effective in capturing long-range dependencies.
The implementation details are explained in Appendix~\ref{app:implementation}. %

\subsubsection{Datasets.}
In our experiments, we consider the Human Connectome Project (HCP) dataset~\cite{van2013wu}, which is a comprehensive publicly available neuroimaging dataset that includes both imaging data and a wide range of behavioral and cognitive data. We process the HCP-Task dataset by parcellating it into 360 distinct brain regions. With respect to other datasets, we use the processed ones from the NeuroGraph benchmark~\cite{said2023neurograph}.
% which employs the established Schaefer~\cite{schaefer2018local} group-level atlases to represent the BOLD signals. These atlases offer a parcellation of the cerebral cortex into hierarchical regions at various levels of granularity (i.e., resolutions). The detailed dataset statistics are provided in Table~\ref{tab:dataset_statistics}.

\subsubsection{Baselines.}
We compare our framework with baselines leveraged by the NeuroGraph benchmark and two state-of-the-art models GraphGPS~\cite{rampavsek2022recipe} and Graph-Mamba~\cite{wang2024graph} that can extract long-range dependencies within the graph data.

% GraphGPS employs a modular framework that integrates SE, PE, MPNN, and a graph transformer, where it allows the replacement of fully-connected Transformer attention with its sparse alternatives.  Graph-Mamba is the pioneering work to applies state space models (SSMs) for non-sequential graph data, where it captures long-range node dependencies with linear time complexity.

% \subsubsection{Implementation.}
% The experiments are implemented with Pytorch 2.0.1 \cite{paszke2019pytorch} on 4 NVIDIA A100 GPUs each with 80GB memory. We set the number of experts $K$, orders $M$, and layers to 3, 2, 3, respectively. We obtain other best hyper-parameters via grid search with the range of learning rate from $10^{-1}$ to $10^{-3}$, and weight decay from $10^{-3}$ to $10^{-5}$, with each configuration run for 100 epochs. 

\begin{table}[!]
\centering
\renewcommand{\arraystretch}{1.}
% \zt{looks confusing, some scores are wanted to be larger, some are smaller. need clearer indicators}}
\resizebox{\columnwidth}{!}{
\begin{tabular}{lcccc}
\toprule
\textbf{Dataset}        & {\textbf{BrainMAP}}          & {w/o LR} & {w/o MoE} & {w/o LB} \\
\midrule

{HCP-Task} $\uparrow$ & \textbf{94.74} &      \underline{94.56}                   &          94.45               &       94.45                   \\
{HCP-Gender} $\uparrow$ &  \textbf{78.92} & 78.09 & 77.58 & \underline{78.50}  \\
{HCP-Age} $\uparrow$ & \textbf{48.44} & \underline{48.13} & 47.81 & 47.54\\
{HCP-FI} $\downarrow$ & \textbf{3.81} & \underline{3.84} & 3.91 & 3.91\\
{HCP-WM} $\downarrow$ & \textbf{10.75} & \underline{10.96} & 11.19 & 11.01 \\
\bottomrule
\end{tabular}
}
\caption{Ablation study of BrainMAP on various datasets. }%BrainMAP consistently surpasses its ablated variants, which demonstrates the effectiveness of all modules.
\label{tab:ablation}
\end{table}

\subsection{Main Results}
To answer \textbf{RQ1}, we first evaluate the performance of BrainMAP in comparison to all baselines on the HCP datasets. We make the following observations from empirical results in Table~\ref{tab:main}. (\romannumeral1) BrainMAP outperforms all baselines across various benchmarks with improvements up to $4.09\%$ over the state-of-the-art, which demonstrates its ability to extract long-range dependencies between brain regions along multiple pathways. (\romannumeral2) BrainMAP and Graph-Mamba surpass traditional GNN models by a large margin, with BrainMAP showing improvements of up to $12.13\%$ on HCP-Task. The observation corroborates the benefit of extracting activation pathways for brain-related tasks. (\romannumeral3) BrainMAP consistently outperforms Graph-Mamba, which demonstrates the effectiveness of learning multiple pathways.

\subsection{Ablation Studies}
To address \textbf{RQ2}, we conduct ablation studies on BrainMAP by removing different components, where \textit{w/o LR} refers to the removal of the structure sequentializer and \textit{w/o LB} indicates the exclusion of the load balancing loss for the MoE. The empirical results in Table~\ref{tab:ablation} lead to the following observations. (\romannumeral1) Both the MoE and the structure sequentializer contribute to the overall performance, which suggests the importance of effective structure sequentialization and the extraction of multiple pathways. (\romannumeral2) The MoE appears to be the most critical component for overall performance, indicating the significant role of learning multiple pathways. (\romannumeral3) The removal of load balancing loss results in a reduction in the overall performance, which illustrates the importance of broad and balanced activation of experts.

\begin{figure}
    \centering
    \includegraphics[width=0.85\columnwidth]{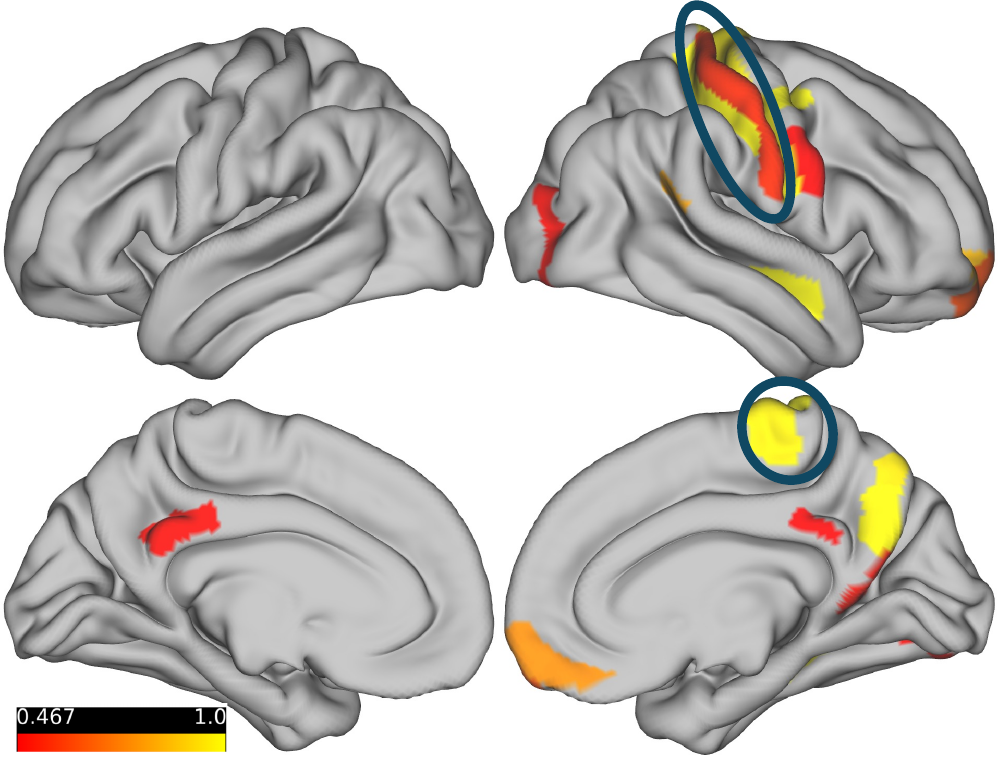}
    \vspace{.1in}
    \caption{Interpretation results of BrainMAP for the task MOTOR in HCP-Task. The average salient regions from random samples. The color bar ranges from 0.4 to 1. The bright-yellow color indicates a high score, while dark-red color indicates a relatively lower score. The ground-truth brain regions given by domain experts are circled in blue.}
    \label{fig:interpret}
\end{figure}

\subsection{Explanation Study}
To answer \textbf{RQ3} and better comprehend the prediction decisions made by different models, we aim to identify the salient brain regions that contribute the most to the predictions. To be more specific, we seek to identify the activated brain regions during a specific task MOTOR from the HCP-Task dataset. We first adopt explanation models to calculate the importance scores of brain regions during the task MOTOR of several random samples from the HCP-Task dataset, where the scores are then averaged to assess the interpretation ability. We adopt the commonly used GNNExplainer~\cite{ying2019gnnexplainer} for ResGCN and GraphGPS, and a Mamba-specific explanation method for Graph-Mamba and BrainMAP to calculate the importance scores. We select the salient brain regions, which are then compared to the ground-truth activated brain region of the HCP-Task given by domain experts, where the correspondence is  measured with Hit@10, Hit@30, and Mean Reciprocal Rank (MRR). The results in Table~\ref{tab:interpret} showcase that BrainMAP achieves higher precision in locating the activated brain regions for the MOTOR task, which demonstrates its reliability and effectiveness. Apart from the quantitative analysis of the interpretation ability of the BrainMAP, we also visualize the interpretation results in Fig.~\ref{fig:interpret}, where top-ranked brain regions of BrainMAP are highlighted with different colors, and ground-truth activated brain regions given by domain experts are circled. We could make the observation that BrainMAP is able to identify several ground-truth regions, which further demonstrates its effectiveness.

\begin{table}[!t]
\centering
\renewcommand{\arraystretch}{1.}
\begin{tabular}{lccc}
\toprule
\textbf{Model}        & Hit@10 &  Hit@30      & MRR \\
\midrule
{ResGCN} & 6.25	& 21.88	& 3.07  \\
{GraphGPS}& 8.75	& 24.38	& 3.13\\
{Graph-Mamba} & \underline{15.00}	& \underline{31.25}	& \underline{6.27}\\
\midrule
\textbf{BrainMAP} & \textbf{19.38}	& \textbf{33.75}	& \textbf{9.26}\\
\bottomrule
\end{tabular}
\caption{Interpretation results of BrainMAP for the task MOTOR from HCP-Task, where the alignment between salient brain regions obtained from different models and the ground truth given by domain experts is measured by three metrics.} %BrainMAP has the best interpretation ability, which demonstrates its effectiveness and reliability.
\label{tab:interpret}
\end{table}

\begin{figure}
    \centering
    \begin{subfigure}[t]{0.49\columnwidth} % subfigure环境，[b]表示对齐方式，0.48\textwidth设置宽度
        \centering
        \includegraphics[width=\textwidth]{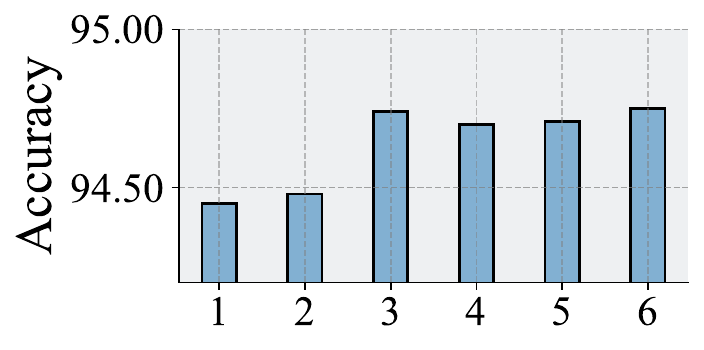} % 插入图形
                \vspace{-0.25in}
        \caption{{\#Experts for HCP-Task}} % 子图的标题
        \label{fig:ratio_a} % 子图的标签，用于交叉引用
    \end{subfigure}
    % \vspace{-.05in}
    \hfill
    \begin{subfigure}[t]{0.49\columnwidth} % subfigure环境，[b]表示对齐方式，0.48\textwidth设置宽度
        \centering
        \includegraphics[width=\textwidth]{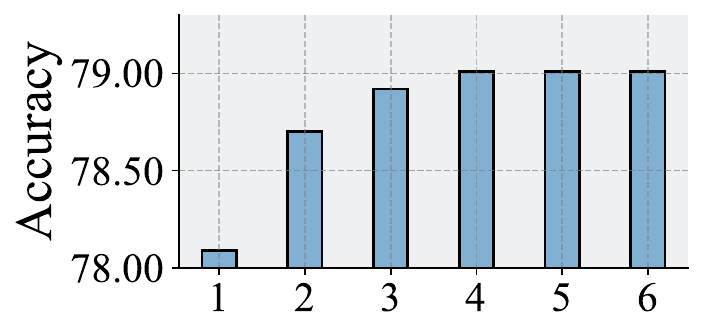} % 插入图形
                        \vspace{-.25in}
        \caption{{\#Experts for HCP-Gender}} % 子图的标题
        \label{fig:ratio_a} % 子图的标签，用于交叉引用
    \end{subfigure}
    \begin{subfigure}[t]{0.49\columnwidth} % subfigure环境，[b]表示对齐方式，0.48\textwidth设置宽度
        \centering
        \includegraphics[width=\textwidth]{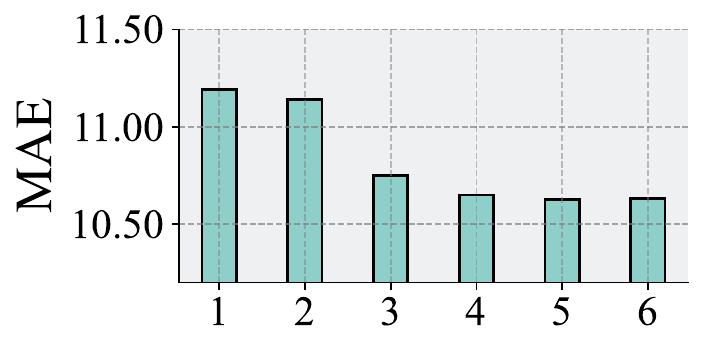} % 插入图形
                \vspace{-.25in}
        \caption{{\#Experts for HCP-WM}} % 子图的标题
        \label{fig:ratio_a} % 子图的标签，用于交叉引用
    \end{subfigure}
    % \vspace{-.05in}
    \hfill
    \begin{subfigure}[t]{0.49\columnwidth} % subfigure环境，[b]表示对齐方式，0.48\textwidth设置宽度
        \centering
        \includegraphics[width=\textwidth]{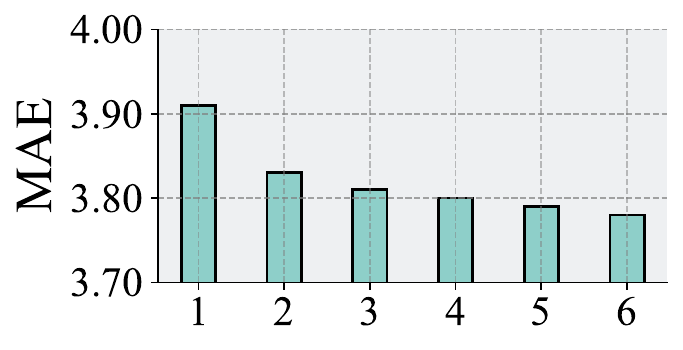} % 插入图形
                \vspace{-.25in}
        \caption{{\#Experts for HCP-FI}} % 子图的标题
        \label{fig:ratio_a} % 子图的标签，用于交叉引用
    \end{subfigure}
    % \vspace{-.05in}
    \caption{The results of varying the number of experts in the MoE on four HCP benchmark datasets. } %It is observed that the accuracy of BrainMAP increases as the number of experts increases up to 3, then remains nearly constant. The results can be attributed to the necessity of the extraction of enough number of pathways.
    \label{fig:MoE}
\end{figure}

\subsection{MoE Analysis}
The MoE is critical in extracting multiple pathways. 
%As a result, the design of the MoE plays an important role in overall performance. 
To answer \textbf{RQ4}, we evaluate the impact of varying the number of experts in the MoE on the model's performance. We could make the following observations from Fig.~\ref{fig:MoE}. (\romannumeral1) The performance of BrainMAP improves as the number of experts increases up to 3. It can be attributed to the fact that experts might be insufficient to extract diverse pathways necessary for comprehensive prediction. (\romannumeral2) The accuracy remains nearly constant once the number of experts exceeds 4, which can be attributed to the limited number of potential pathways in the brain. 
To gain a deeper understanding of the MoE component, we analyze the activation distribution across different layers of BrainMAP, as shown in Fig.~\ref{fig:activation}. The results illustrated that BrainMAP consistently maintains high activation rates, with a minimum of $66.7\%$ activation rate for HCP-Gender. The findings further suggest that BrainMAP effectively extracts multiple pathways, as evidenced by the activation of diverse experts.

\begin{figure}
    \centering
    \includegraphics[width=0.95\columnwidth]{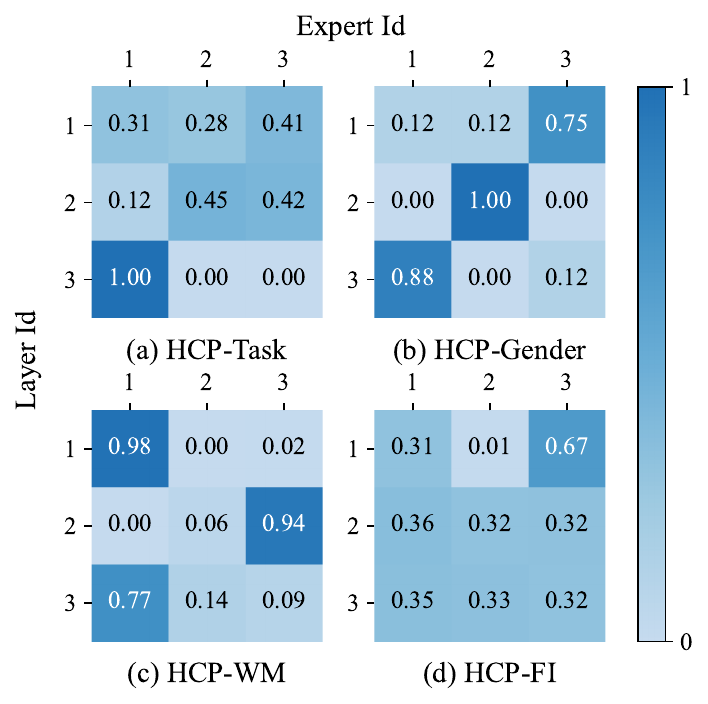}
    \caption{The activation distribution of the experts across different model layers. BrainMAP consistently maintains high activation rates on various HCP datasets.}
    \label{fig:activation}
\end{figure}

\section{Conclusion}
Despite significant progress has been made in understanding brain activity through functional connectivity (FC) graphs, challenges persist in effectively capturing and interpreting the complex, long-range dependencies and multiple pathways that are inherent in these graphs. 
In this work, we introduce BrainMAP, a novel framework designed to extract multiple long-range activation pathways with adaptive sequentialization and pathway aggregation. Experiments demonstrate the effectiveness of BrainMAP in extracting underlying activation pathways for predictions tasks.

\section{Acknowledgments}
This work is supported in part by the National Science Foundation under grants IIS-2006844, IIS-2144209, IIS-2223769, IIS-2331315, CNS-2154962, BCS-2228534, and CMMI-2411248, the Commonwealth Cyber Initiative Awards under grants VV-1Q24-011, VV-1Q25-004, and the research gift funding from Netflix and Snap. Xinyu Zhao and Tianlong Chen are supported by NIH OT2OD038045-01 and UNC SDSS Seed Grant.

\appendix
% \section{Hyperparameter Tuning}
% \section{Evaluation Details}

% \section{Acknowledgment of AI Assistance in Writing and Revision}
% We utilized ChatGPT-4 for revising and enhancing the wording
% of this paper.

% \section{Reproducibility}
% The implmentation of our framework, including code for
% model construction, data preprocessing, and experiments, is released at xxx.

% \section{Pseudo-Code Style Description of xxx}
% \label{app:code}
% A pseudo-code style description of xxx is attached here.

% \section*{Ethical Statement}

\bibliography{aaai24}

\begin{thebibliography}{61}
\providecommand{\natexlab}[1]{#1}

\bibitem[{Bassett and Sporns(2017)}]{bassett2017network}
Bassett, D.~S.; and Sporns, O. 2017.
\newblock Network neuroscience.
\newblock \emph{Nature neuroscience}, 20(3): 353--364.

\bibitem[{Chang, Lin, and Lane(2021)}]{chang2021machine}
Chang, C.-H.; Lin, C.-H.; and Lane, H.-Y. 2021.
\newblock Machine learning and novel biomarkers for the diagnosis of Alzheimer’s disease.
\newblock \emph{International journal of molecular sciences}, 22(5): 2761.

\bibitem[{Chen et~al.(2024)Chen, Qi, Wang, and Pan}]{article}
Chen, J.; Qi, Y.; Wang, Y.; and Pan, G. 2024.
\newblock Bridging the Semantic Latent Space between Brain and Machine: Similarity Is All You Need.
\newblock \emph{AAAI}, 38: 11302--11310.

\bibitem[{Cui et~al.(2022{\natexlab{a}})Cui, Dai, Zhu, Kan, Gu, Lukemire, Zhan, He, Guo, and Yang}]{cui2022braingb}
Cui, H.; Dai, W.; Zhu, Y.; Kan, X.; Gu, A. A.~C.; Lukemire, J.; Zhan, L.; He, L.; Guo, Y.; and Yang, C. 2022{\natexlab{a}}.
\newblock Braingb: a benchmark for brain network analysis with graph neural networks.
\newblock \emph{IEEE transactions on medical imaging}.

\bibitem[{Cui et~al.(2022{\natexlab{b}})Cui, Lu, Li, and Yang}]{cui2022positional}
Cui, H.; Lu, Z.; Li, P.; and Yang, C. 2022{\natexlab{b}}.
\newblock On positional and structural node features for graph neural networks on non-attributed graphs.
\newblock In \emph{Proceedings of the 31st ACM International Conference on Information \& Knowledge Management}, 3898--3902.

\bibitem[{Dahan et~al.(2021)Dahan, Williams, Rueckert, and Robinson}]{dahan2021improving}
Dahan, S.; Williams, L.~Z.; Rueckert, D.; and Robinson, E.~C. 2021.
\newblock Improving phenotype prediction using long-range spatio-temporal dynamics of functional connectivity.
\newblock In \emph{Machine Learning in Clinical Neuroimaging: 4th International Workshop}.

\bibitem[{Dauphin et~al.(2017)Dauphin, Fan, Auli, and Grangier}]{dauphin2017language}
Dauphin, Y.~N.; Fan, A.; Auli, M.; and Grangier, D. 2017.
\newblock Language modeling with gated convolutional networks.
\newblock In \emph{International conference on machine learning}.

\bibitem[{Davis et~al.(2020)Davis, Aghaeepour, Ahn, Angst, Borsook, Brenton, Burczynski, Crean, Edwards, Gaudilliere et~al.}]{davis2020discovery}
Davis, K.~D.; Aghaeepour, N.; Ahn, A.~H.; Angst, M.~S.; Borsook, D.; Brenton, A.; Burczynski, M.~E.; Crean, C.; Edwards, R.; Gaudilliere, B.; et~al. 2020.
\newblock Discovery and validation of biomarkers to aid the development of safe and effective pain therapeutics: challenges and opportunities.
\newblock \emph{Nature Reviews Neurology}, 16(7): 381--400.

\bibitem[{Eslami et~al.(2019)Eslami, Mirjalili, Fong, Laird, and Saeed}]{eslami2019asd}
Eslami, T.; Mirjalili, V.; Fong, A.; Laird, A.~R.; and Saeed, F. 2019.
\newblock ASD-DiagNet: a hybrid learning approach for detection of autism spectrum disorder using fMRI data.
\newblock \emph{Frontiers in neuroinformatics}, 13: 70.

\bibitem[{Finn, Poldrack, and Shine(2023)}]{finn2023functional}
Finn, E.~S.; Poldrack, R.~A.; and Shine, J.~M. 2023.
\newblock Functional neuroimaging as a catalyst for integrated neuroscience.
\newblock \emph{Nature}, 623(7986): 263--273.

\bibitem[{Fox and Raichle(2007)}]{fox2007spontaneous}
Fox, M.~D.; and Raichle, M.~E. 2007.
\newblock Spontaneous fluctuations in brain activity observed with functional magnetic resonance imaging.
\newblock \emph{Nature reviews neuroscience}.

\bibitem[{Gao et~al.(2024)Gao, Zhang, Li, Tang, Liu, Zhou, Ying, Zhu, and Zhang}]{gao2024local}
Gao, D.; Zhang, H.; Li, P.; Tang, T.; Liu, S.; Zhou, Z.; Ying, S.; Zhu, Y.; and Zhang, Y. 2024.
\newblock A Local-Ascending-Global Learning Strategy for Brain-Computer Interface.
\newblock In \emph{AAAI}, volume~38, 10039--10047.

\bibitem[{Gu and Dao(2023)}]{gu2023mamba}
Gu, A.; and Dao, T. 2023.
\newblock Mamba: Linear-time sequence modeling with selective state spaces.
\newblock \emph{arXiv preprint arXiv:2312.00752}.

\bibitem[{He et~al.(2020)He, Kong, Holmes, Nguyen, Sabuncu, Eickhoff, Bzdok, Feng, and Yeo}]{he2020deep}
He, T.; Kong, R.; Holmes, A.~J.; Nguyen, M.; Sabuncu, M.~R.; Eickhoff, S.~B.; Bzdok, D.; Feng, J.; and Yeo, B.~T. 2020.
\newblock Deep neural networks and kernel regression achieve comparable accuracies for functional connectivity prediction of behavior and demographics.
\newblock \emph{NeuroImage}.

\bibitem[{Hsu et~al.(2024)Hsu, Cox, Xu, Tan, Zhai, Hu, Pratt, Chen, Hu, and Ding}]{hsu2024thought}
Hsu, C.-Y.; Cox, K.; Xu, J.; Tan, Z.; Zhai, T.; Hu, M.; Pratt, D.; Chen, T.; Hu, Z.; and Ding, Y. 2024.
\newblock Thought Graph: Generating Thought Process for Biological Reasoning.
\newblock In \emph{Companion Proceedings of the ACM on Web Conference 2024}, 537--540.

\bibitem[{Jacobs et~al.(1991)Jacobs, Jordan, Nowlan, and Hinton}]{jacobs1991adaptive}
Jacobs, R.~A.; Jordan, M.~I.; Nowlan, S.~J.; and Hinton, G.~E. 1991.
\newblock Adaptive mixtures of local experts.
\newblock \emph{Neural computation}, 3(1): 79--87.

\bibitem[{Jiang et~al.(2024)Jiang, Sablayrolles, Roux, Mensch, Savary, Bamford, Chaplot, Casas, Hanna, Bressand et~al.}]{jiang2024mixtral}
Jiang, A.~Q.; Sablayrolles, A.; Roux, A.; Mensch, A.; Savary, B.; Bamford, C.; Chaplot, D.~S.; Casas, D. d.~l.; Hanna, E.~B.; Bressand, F.; et~al. 2024.
\newblock Mixtral of experts.
\newblock \emph{arXiv preprint arXiv:2401.04088}.

\bibitem[{Jo, Nho, and Saykin(2019)}]{jo2019deep}
Jo, T.; Nho, K.; and Saykin, A.~J. 2019.
\newblock Deep learning in Alzheimer's disease: diagnostic classification and prognostic prediction using neuroimaging data.
\newblock \emph{Frontiers in aging neuroscience}, 11: 220.

\bibitem[{Jordan and Jacobs(1994)}]{jordan1994hierarchical}
Jordan, M.~I.; and Jacobs, R.~A. 1994.
\newblock Hierarchical mixtures of experts and the EM algorithm.
\newblock \emph{Neural computation}.

\bibitem[{Kan et~al.(2022)Kan, Dai, Cui, Zhang, Guo, and Yang}]{kan2022brain}
Kan, X.; Dai, W.; Cui, H.; Zhang, Z.; Guo, Y.; and Yang, C. 2022.
\newblock Brain network transformer.
\newblock \emph{NeurIPS}, 35: 25586--25599.

\bibitem[{Kawahara et~al.(2017)Kawahara, Brown, Miller, Booth, Chau, Grunau, Zwicker, and Hamarneh}]{kawahara2017brainnetcnn}
Kawahara, J.; Brown, C.~J.; Miller, S.~P.; Booth, B.~G.; Chau, V.; Grunau, R.~E.; Zwicker, J.~G.; and Hamarneh, G. 2017.
\newblock BrainNetCNN: Convolutional neural networks for brain networks; towards predicting neurodevelopment.
\newblock \emph{NeuroImage}, 146: 1038--1049.

\bibitem[{Keller, Taube, and Lauber(2018)}]{keller2018task}
Keller, M.; Taube, W.; and Lauber, B. 2018.
\newblock Task-dependent activation of distinct fast and slow (er) motor pathways during motor imagery.
\newblock \emph{Brain stimulation}, 11(4): 782--788.

\bibitem[{Kim, Ye, and Kim(2021)}]{kim2021learning}
Kim, B.-H.; Ye, J.~C.; and Kim, J.-J. 2021.
\newblock Learning dynamic graph representation of brain connectome with spatio-temporal attention.
\newblock \emph{NeurIPS}, 34: 4314--4327.

\bibitem[{Kipf and Welling(2017)}]{kipf2017semi}
Kipf, T.~N.; and Welling, M. 2017.
\newblock Semi-supervised classification with graph convolutional networks.
\newblock In \emph{ICLR}.

\bibitem[{Kohoutov{\'a} et~al.(2020)Kohoutov{\'a}, Heo, Cha, Lee, Moon, Wager, and Woo}]{kohoutova2020toward}
Kohoutov{\'a}, L.; Heo, J.; Cha, S.; Lee, S.; Moon, T.; Wager, T.~D.; and Woo, C.-W. 2020.
\newblock Toward a unified framework for interpreting machine-learning models in neuroimaging.
\newblock \emph{Nature protocols}, 15(4): 1399--1435.

\bibitem[{Lee et~al.(2016)Lee, Bonin, Reed, Graham, Hood, Glattfelder, and Reid}]{lee2016anatomy}
Lee, W.-C.~A.; Bonin, V.; Reed, M.; Graham, B.~J.; Hood, G.; Glattfelder, K.; and Reid, R.~C. 2016.
\newblock Anatomy and function of an excitatory network in the visual cortex.
\newblock \emph{Nature}, 532(7599): 370--374.

\bibitem[{Li and Fan(2019)}]{li2019interpretable}
Li, H.; and Fan, Y. 2019.
\newblock Interpretable, highly accurate brain decoding of subtly distinct brain states from functional MRI using intrinsic functional networks and long short-term memory recurrent neural networks.
\newblock \emph{NeuroImage}.

\bibitem[{Li et~al.(2021)Li, Zhou, Dvornek, Zhang, Gao, Zhuang, Scheinost, Staib, Ventola, and Duncan}]{li2021braingnn}
Li, X.; Zhou, Y.; Dvornek, N.; Zhang, M.; Gao, S.; Zhuang, J.; Scheinost, D.; Staib, L.~H.; Ventola, P.; and Duncan, J.~S. 2021.
\newblock Braingnn: Interpretable brain graph neural network for fmri analysis.
\newblock \emph{Medical Image Analysis}, 74: 102233.

\bibitem[{Li et~al.(2022)Li, Zhang, Nie, Zhang, Fang, Xu, Wu, Hu, Wang, Zhang et~al.}]{li2022brain}
Li, Y.; Zhang, X.; Nie, J.; Zhang, G.; Fang, R.; Xu, X.; Wu, Z.; Hu, D.; Wang, L.; Zhang, H.; et~al. 2022.
\newblock Brain connectivity based graph convolutional networks and its application to infant age prediction.
\newblock \emph{IEEE transactions on medical imaging}, 41(10): 2764--2776.

\bibitem[{Liao, Wan, and Du(2024)}]{liao2024joint}
Liao, M.; Wan, G.; and Du, B. 2024.
\newblock Joint Learning Neuronal Skeleton and Brain Circuit Topology with Permutation Invariant Encoders for Neuron Classification.
\newblock In \emph{AAAI}.

\bibitem[{Liu et~al.(2023)Liu, Zhou, Shi, Du, Zhao, Wu, Liu, Liu, and Hu}]{liu2023coupling}
Liu, X.; Zhou, M.; Shi, G.; Du, Y.; Zhao, L.; Wu, Z.; Liu, D.; Liu, T.; and Hu, X. 2023.
\newblock Coupling artificial neurons in bert and biological neurons in the human brain.
\newblock In \emph{AAAI}.

\bibitem[{Luo et~al.(2020)Luo, Cheng, Xu, Yu, Zong, Chen, and Zhang}]{luo2020parameterized}
Luo, D.; Cheng, W.; Xu, D.; Yu, W.; Zong, B.; Chen, H.; and Zhang, X. 2020.
\newblock Parameterized explainer for graph neural network.
\newblock \emph{NeurIPS}, 33: 19620--19631.

\bibitem[{Morris et~al.(2019)Morris, Ritzert, Fey, Hamilton, Lenssen, Rattan, and Grohe}]{morris2019weisfeiler}
Morris, C.; Ritzert, M.; Fey, M.; Hamilton, W.~L.; Lenssen, J.~E.; Rattan, G.; and Grohe, M. 2019.
\newblock Weisfeiler and leman go neural: Higher-order graph neural networks.
\newblock In \emph{AAAI}, volume~33, 4602--4609.

\bibitem[{Paszke et~al.(2019)Paszke, Gross, Massa, Lerer, Bradbury, Chanan, Killeen, Lin, Gimelshein, Antiga et~al.}]{paszke2019pytorch}
Paszke, A.; Gross, S.; Massa, F.; Lerer, A.; Bradbury, J.; Chanan, G.; Killeen, T.; Lin, Z.; Gimelshein, N.; Antiga, L.; et~al. 2019.
\newblock Pytorch: An imperative style, high-performance deep learning library.
\newblock \emph{NeurIPS}.

\bibitem[{Ramp{\'a}{\v{s}}ek et~al.(2022)Ramp{\'a}{\v{s}}ek, Galkin, Dwivedi, Luu, Wolf, and Beaini}]{rampavsek2022recipe}
Ramp{\'a}{\v{s}}ek, L.; Galkin, M.; Dwivedi, V.~P.; Luu, A.~T.; Wolf, G.; and Beaini, D. 2022.
\newblock Recipe for a general, powerful, scalable graph transformer.
\newblock \emph{NeurIPS}, 35: 14501--14515.

\bibitem[{Said et~al.(2023)Said, Bayrak, Derr, Shabbir, Moyer, Chang, and Koutsoukos}]{said2023neurograph}
Said, A.; Bayrak, R.; Derr, T.; Shabbir, M.; Moyer, D.; Chang, C.; and Koutsoukos, X. 2023.
\newblock Neurograph: Benchmarks for graph machine learning in brain connectomics.
\newblock \emph{NeurIPS}.

\bibitem[{Sankar et~al.(2018)Sankar, Wu, Gou, Zhang, and Yang}]{sankar2018dynamic}
Sankar, A.; Wu, Y.; Gou, L.; Zhang, W.; and Yang, H. 2018.
\newblock Dynamic graph representation learning via self-attention networks.
\newblock \emph{arXiv preprint arXiv:1812.09430}.

\bibitem[{Shazeer et~al.(2017)Shazeer, Mirhoseini, Maziarz, Davis, Le, Hinton, and Dean}]{shazeer2017outrageously}
Shazeer, N.; Mirhoseini, A.; Maziarz, K.; Davis, A.; Le, Q.; Hinton, G.; and Dean, J. 2017.
\newblock Outrageously large neural networks: The sparsely-gated mixture-of-experts layer.
\newblock \emph{arXiv preprint arXiv:1701.06538}.

\bibitem[{Spearman(1904)}]{spearman1904proof}
Spearman, C. 1904.
\newblock The Proof and Measurement of Association between Two Things.
\newblock \emph{The American Journal of Psychology}.

\bibitem[{Sporns(2011)}]{sporns2011human}
Sporns, O. 2011.
\newblock The human connectome: a complex network.
\newblock \emph{Annals of the new York Academy of Sciences}.

\bibitem[{Tan et~al.(2022{\natexlab{a}})Tan, Ding, Guo, and Liu}]{tan2022supervised}
Tan, Z.; Ding, K.; Guo, R.; and Liu, H. 2022{\natexlab{a}}.
\newblock Supervised graph contrastive learning for few-shot node classification.
\newblock In \emph{Joint European Conference on Machine Learning and Knowledge Discovery in Databases}.

\bibitem[{Tan et~al.(2022{\natexlab{b}})Tan, {Wang}, Ding, Li, and Liu}]{tan2022transductive}
Tan, Z.; {Wang}, S.; Ding, K.; Li, J.; and Liu, H. 2022{\natexlab{b}}.
\newblock Transductive Linear Probing: A Novel Framework for Few-Shot Node Classification.
\newblock In \emph{LoG}.

\bibitem[{Thiebaut~de Schotten and Forkel(2022)}]{thiebaut2022emergent}
Thiebaut~de Schotten, M.; and Forkel, S.~J. 2022.
\newblock The emergent properties of the connected brain.
\newblock \emph{Science}.

\bibitem[{Thomas, R{\'e}, and Poldrack(2022)}]{thomas2022interpreting}
Thomas, A.~W.; R{\'e}, C.; and Poldrack, R.~A. 2022.
\newblock Interpreting mental state decoding with deep learning models.
\newblock \emph{Trends in Cognitive Sciences}, 26(11): 972--986.

\bibitem[{Van~Essen et~al.(2013)Van~Essen, Smith, Barch, Behrens, Yacoub, Ugurbil, Consortium et~al.}]{van2013wu}
Van~Essen, D.~C.; Smith, S.~M.; Barch, D.~M.; Behrens, T.~E.; Yacoub, E.; Ugurbil, K.; Consortium, W.-M.~H.; et~al. 2013.
\newblock The WU-Minn human connectome project: an overview.
\newblock \emph{Neuroimage}.

\bibitem[{Vaswani et~al.(2017)Vaswani, Shazeer, Parmar, Uszkoreit, Jones, Gomez, Kaiser, and Polosukhin}]{vaswani2017attention}
Vaswani, A.; Shazeer, N.; Parmar, N.; Uszkoreit, J.; Jones, L.; Gomez, A.~N.; Kaiser, {\L}.; and Polosukhin, I. 2017.
\newblock Attention is all you need.
\newblock \emph{NeurIPS}, 30.

\bibitem[{Veli{\v{c}}kovi{\'c} et~al.(2018)Veli{\v{c}}kovi{\'c}, Cucurull, Casanova, Romero, Lio, and Bengio}]{velivckovic2017graph}
Veli{\v{c}}kovi{\'c}, P.; Cucurull, G.; Casanova, A.; Romero, A.; Lio, P.; and Bengio, Y. 2018.
\newblock Graph attention networks.
\newblock In \emph{ICLR}.

\bibitem[{Wang et~al.(2024{\natexlab{a}})Wang, Tsepa, Ma, and Wang}]{wang2024graph}
Wang, C.; Tsepa, O.; Ma, J.; and Wang, B. 2024{\natexlab{a}}.
\newblock Graph-mamba: Towards long-range graph sequence modeling with selective state spaces.
\newblock \emph{arXiv preprint arXiv:2402.00789}.

\bibitem[{Wang et~al.(2024{\natexlab{b}})Wang, An, Cheng, Zhou, Hwang, and Hsieh}]{wang2024one}
Wang, R.; An, S.; Cheng, M.; Zhou, T.; Hwang, S.~J.; and Hsieh, C.-J. 2024{\natexlab{b}}.
\newblock One Prompt is not Enough: Automated Construction of a Mixture-of-Expert Prompts.
\newblock \emph{arXiv preprint arXiv:2407.00256}.

\bibitem[{{Wang}, Chen, and Li(2022)}]{wang2022glitter}
{Wang}, S.; Chen, C.; and Li, J. 2022.
\newblock Graph Few-shot Learning with Task-specific Structures.
\newblock In \emph{NeurIPS}.

\bibitem[{{Wang} et~al.(2024){Wang}, Chen, Shi, Shen, and Li}]{wang2024mixture}
{Wang}, S.; Chen, Z.; Shi, C.; Shen, C.; and Li, J. 2024.
\newblock Mixture of Demonstrations for In-Context Learning.
\newblock In \emph{NeurIPS}.

\bibitem[{{Wang} et~al.(2023){Wang}, Tan, Liu, and Li}]{wang2023contrast}
{Wang}, S.; Tan, Z.; Liu, H.; and Li, J. 2023.
\newblock Contrastive Meta-Learning for Few-shot Node Classification.
\newblock In \emph{SIGKDD}.

\bibitem[{Wang et~al.(2022)Wang, Yao, Rekik, and Zhang}]{wang2022contrastive}
Wang, X.; Yao, L.; Rekik, I.; and Zhang, Y. 2022.
\newblock Contrastive functional connectivity graph learning for population-based fMRI classification.
\newblock In \emph{International Conference on Medical Image Computing and Computer-Assisted Intervention}, 221--230. Springer.

\bibitem[{Xu et~al.(2024)Xu, Peng, Shi, Hua, and Dong}]{xu2024learning}
Xu, Y.; Peng, Z.; Shi, B.; Hua, X.; and Dong, B. 2024.
\newblock Learning dynamic graph representations through timespan view contrasts.
\newblock \emph{Neural Networks}, 176: 106384.

\bibitem[{Xu et~al.(2023)Xu, Shi, Ma, Dong, Zhou, and Zheng}]{xu2023cldg}
Xu, Y.; Shi, B.; Ma, T.; Dong, B.; Zhou, H.; and Zheng, Q. 2023.
\newblock CLDG: Contrastive learning on dynamic graphs.
\newblock In \emph{2023 IEEE 39th International Conference on Data Engineering (ICDE)}, 696--707. IEEE.

\bibitem[{Yang et~al.(2022)Yang, Zhu, Cui, Kan, He, Guo, and Yang}]{yang2022data}
Yang, Y.; Zhu, Y.; Cui, H.; Kan, X.; He, L.; Guo, Y.; and Yang, C. 2022.
\newblock Data-efficient brain connectome analysis via multi-task meta-learning.
\newblock In \emph{Proceedings of the 28th ACM SIGKDD Conference on Knowledge Discovery and Data Mining}, 4743--4751.

\bibitem[{Ying et~al.(2019)Ying, Bourgeois, You, Zitnik, and Leskovec}]{ying2019gnnexplainer}
Ying, Z.; Bourgeois, D.; You, J.; Zitnik, M.; and Leskovec, J. 2019.
\newblock Gnnexplainer: Generating explanations for graph neural networks.
\newblock \emph{NeurIPS}.

\bibitem[{You et~al.(2020)You, Chen, Sui, Chen, Wang, and Shen}]{you2020graph}
You, Y.; Chen, T.; Sui, Y.; Chen, T.; Wang, Z.; and Shen, Y. 2020.
\newblock Graph contrastive learning with augmentations.
\newblock \emph{NeurIPS}.

\bibitem[{Zhang et~al.(2022)Zhang, Wang, Lin, Zhang, and Zong}]{zhang2022probing}
Zhang, X.; Wang, S.; Lin, N.; Zhang, J.; and Zong, C. 2022.
\newblock Probing word syntactic representations in the brain by a feature elimination method.
\newblock In \emph{AAAI}, volume~36, 11721--11729.

\bibitem[{Zhang, Ji, and Liu(2024)}]{zhang2024metarlec}
Zhang, Z.; Ji, J.; and Liu, J. 2024.
\newblock MetaRLEC: Meta-Reinforcement Learning for Discovery of Brain Effective Connectivity.
\newblock In \emph{AAAI}.

\bibitem[{Zhou et~al.(2020)Zhou, Cui, Hu, Zhang, Yang, Liu, Wang, Li, and Sun}]{zhou2020graph}
Zhou, J.; Cui, G.; Hu, S.; Zhang, Z.; Yang, C.; Liu, Z.; Wang, L.; Li, C.; and Sun, M. 2020.
\newblock Graph neural networks: A review of methods and applications.
\newblock \emph{AI open}, 1: 57--81.

\end{thebibliography}

\newpage
\appendix

\section{Proof of Theorem 4.1}\label{app:theorem1}
\begin{customthm}{4.1}
    The mean and standard deviation of $S_i$ are the same as those of any real rank variable $R$ for a sample size of $N$, i.e., 
    \begin{equation}
        \mu(S_i)=\mu(R), \quad \sigma(S_i)=\sigma(R).
    \end{equation}
    \label{theorem1}
\end{customthm}

    \begin{proof}
We start by showing that $\widetilde{S}_i=\frac{s_i - \mathbb{E}[s] }{\sqrt{\mathbb{E}[s^2] - (\mathbb{E}[s])^2}}$ is the standardized value of $s_i$, i.e., $\mu(\widetilde{S}_i)=0$ and $\sigma(\widetilde{S}_i)=1$.
\begin{equation}
    \mu(\widetilde{S}_i)=\mathbb{E}\left[\frac{s_i - \mathbb{E}[s] }{\sqrt{\mathbb{E}[s^2] - (\mathbb{E}[s])^2}}\right]=\frac{\mathbb{E}[s]-\mathbb{E}[s]}{\sqrt{\mathbb{E}[s^2] - (\mathbb{E}[s])^2}}=0
\end{equation}
\begin{equation}
\begin{aligned}
       \sigma^2(\widetilde{S}_i)
       &=\mu(\widetilde{S}_i^2)-\mu^2(\widetilde{S}_i)\\
       &=\mu(\widetilde{S}_i^2)\\
       &=\mathbb{E}\left[\frac{(s_i-\mathbb{E}[s])^2}{\mathbb{E}[s^2] - (\mathbb{E}[s])^2}\right]\\
       &=\frac{\mathbb{E}[(s_i-\mathbb{E}[s])^2]}{\mathbb{E}[s^2] - (\mathbb{E}[s])^2}\\
       &=\frac{\mathbb{E}[s^2] - (\mathbb{E}[s])^2}{\mathbb{E}[s^2] - (\mathbb{E}[s])^2}\\
       &=1
\end{aligned}
\end{equation}
For a real rank variable $R$ of sample size $N$, we denote $R_i$ as the corresponding rank of the $i$-th sample. By definition of rankings, we know $1\leq R_i\leq N$, $i=1,2,\dotsc, N$, and all $R_i$ are distinct integers. We calculate the mean and standard deviation of $R$ as follows. Firstly, by definition of rankings,  we can consider $R$ as a random variable that is uniformly distributed on $\{1,2,\dotsc, N\}$. Thus, we can obtain
\begin{equation}
    \mu(R)=\mathbb{E}[R]=\frac{1}{N} \sum_{i=1}^N N=\frac{(N+1)}{2}, 
    \end{equation}
    \begin{equation}
    \mathbb{E}\left[R^2\right]=\frac{1}{N} \sum_{i=1}^N i^2=\frac{(N+1)(2 N+1)}{6}.
\end{equation}
Therefore, 
\begin{equation}
\begin{aligned}
    \sigma^2(R)&
    =\mathbb{E}\left[R^2\right]-(\mathbb{E}[R])^2 \\
    &=\frac{(N+1)(2 N+1)}{6} - \left(\frac{(N+1)}{2}\right)^2\\
    &=\frac{N^2-1}{12}.
    \end{aligned}
\end{equation}
Therefore, with the calculated mean and variance of $R$, we can rewrite Eq. (4) as follows:
\begin{equation}
    S_i=\widetilde{S}_i\cdot\sigma(R)+\mu(R).
\end{equation}
Since $\mu(\widetilde{S}_i)=0$ and $\sigma(\widetilde{S}_i)=1$, we know the linear transformation of $\widetilde{S}_i$ will accordingly change the mean and variance. Hence, we have $\mu(S_i)=\mu(R)$ and $\sigma(S_i)=\sigma(R)$.
\end{proof}

\begin{table*}[ht]
\centering
		\setlength\tabcolsep{9pt}%调列距
\renewcommand{\arraystretch}{1}
\begin{tabular}{lcccccc}
\toprule
\textbf{Dataset} & {HCP-Task} & {HCP-Gender} & {HCP-Age} & {HCP-FI} & {HCP-WM} \\
%\textbf{Evaluation Metric  }             & {Accuracy $\uparrow$} & {Accuracy $\uparrow$} & {Accuracy $\uparrow$} & {MAE $\downarrow$} & {MAE $\downarrow$} \\
\midrule
GCN            & 3.17~$(\pm 0.06)$	& 8.09~$(\pm 9.17)$	& 3.13~$(\pm 2.56)$	& 3.98~$(\pm 3.52)$	& 3.52~$(\pm 5.18)$ \\
GAT            & 6.47~$(\pm 0.01)$	& 9.27~$(\pm 0.14)$	& 2.62~$(\pm 0.08)$	& 4.06~$(\pm 0.76)$	& 4.05~$(\pm 1.61)$ \\
SAGE     & 2.49~$(\pm 0.01)$	& 10.50~$(\pm 3.72)$	& 4.44~$(\pm 0.03)$	& 2.03~$(\pm 0.29)$	& 1.82~$(\pm 0.04)$ \\
ResGCN & 22.94~$(\pm 0.03)$	& 20.82~$(\pm 1.25)$	& 23.04~$(\pm 0.69)$	& 21.45~$(\pm 1.92)$	& 15.78~$(\pm 0.46)$ \\
GraphGPS  & 29.37~$(\pm 3.65)$	& 22.54~$(\pm 0.13)$	& 16.83~$(\pm 0.16)$	& 21.22~$(\pm 0.14)$	& 14.44~$(\pm 0.78)$ \\
Graph-Mamba   & 17.78~$(\pm 1.10)$	& 14.71~$(\pm 0.40)$	& 6.22~$(\pm 0.08)$	& 5.24~$(\pm 0.01)$	& 6.61~$(\pm 0.14)$ \\
{BrainMAP}  & 34.93~$(\pm 2.73)$	& 16.91~$(\pm 0.30)$	& 12.27~$(\pm 0.01)$	& 9.23~$(\pm 0.02)$	& 12.62~$(\pm 0.01)$ \\
\bottomrule
\end{tabular}
\caption{The average training time per epoch (in seconds) on 4 A100 GPUs.} %It is observed that BrainMAP consistently outperforms baselines on all HCP datasets, which demonstrates its effectiveness on brain-related tasks.
\label{tab:main}
\end{table*}

\section{Proof of Theorem 4. 2}\label{app:theorem2}
\begin{customthm}{4.2}
    Minimizing the loss $\mathcal{L}$ described in Eq.~(6) equals maximizing the Spearman's rank correlation coefficient~\cite{spearman1904proof} between learned orders and good orders, while minimizing the coefficient between learned orders and bad orders.
    \begin{equation}
       \min \mathcal{L} (\hat{\phi},{\Phi}_p, {\Phi}_n)  \equiv \max \frac{\sum\limits_{\boldsymbol{\phi}\in{\Phi}_n}\left(1-r(\hat{\phi},\phi)\right)}{\sum\limits_{\boldsymbol{\phi}\in{\Phi}_p}\left(1-r(\hat{\phi},\phi)\right)},
    \end{equation}
    where the coefficient is calculated as:
    \begin{equation}
        r(\hat{\phi},\phi)=\rho_{S,S^\phi}=\frac{\text{cov}(S, S^\phi)}{\sigma_S \sigma_{S^\phi}}.
    \end{equation}
    \label{theorem2}
\end{customthm}
\begin{proof}
We start by proving that minimizing $\sum\limits_{i=1}^N (S_i- R_i)^2$ equals maximizing the Spearman's rank correlation coefficient~\cite{spearman1904proof} between $S$ and $R$
\begin{equation}
    r_s = \rho_{S,R}=\frac{\text{cov}(S, R)}{\sigma_S \sigma_{R}},
\end{equation}
where $R$ is a rank variable with a sample size of $N$, and $R_i$ is the corresponding rank of the $i$-th sample. Therefore, we have $1\leq R_i\leq N$, $i=1,2,\dotsc, N$, and all $R_i$ are distinct integers. Moreover, $\text{cov}(S, R)$ is the covariance of $S$ and $R$, and $\sigma_S$ and $\sigma_{R}$ are their standard deviations. Particularly, the covariance can be calculated as follows:
\begin{equation}
\begin{aligned} 
\text{cov}(S, R)
& =\mathbb{E}[SR]-\mathbb{E}[S]\mathbb{E}[R]\\
&=\frac{1}{N} \sum_{i=1}^N S_i R_i-\overline{S} \overline{R}\\ 
& =\frac{1}{N} \sum_{i=1}^N \frac{1}{2}\left(S_i^2+R_i^2-(S_i-R_i)^2\right)-\overline{S} \overline{R}\\ 
& =\frac{1}{2} \frac{1}{N} \sum_{i=1}^N R_i^2+\frac{1}{2} \frac{1}{N} \sum_{i=1}^N S_i^2-\frac{1}{2 N} \sum_{i=1}^N (S_i-R_i)^2\\
&\quad-\overline{S} \overline{R} \\ 
& =\frac{1}{2n} \sum_{i=1}^N(S_i^2+R_i^2)-\overline{S} \overline{R}-\frac{1}{2 N} \sum_{i=1}^N (S_i-R_i)^2 \\ 
\end{aligned}
\end{equation}
According to Theorem~\ref{theorem1}, we know $S$ and $R$ have the same mean and variance. Therefore, we know
\begin{equation}
    \frac{1}{N}\sum_{i=1}^N S_i^2=\frac{1}{N}\sum_{i=1}^N R_i^2, \quad \overline{S}=\overline{R}.
\end{equation}
In this way, we rewrite $\text{cov}(S, R) $ as follows:
\begin{equation}
    \begin{aligned} 
\text{cov}(S, R) 
& =\frac{1}{2N} \sum_{i=1}^N(S_i^2+R_i^2)-\overline{S} \overline{R}-\frac{1}{2 N} \sum_{i=1}^N (S_i-R_i)^2\\
& =\left(\frac{1}{N} \sum_{i=1}^NR_i^2-(\overline{R})^2\right)-\frac{1}{2 N} \sum_{i=1}^N (S_i-R_i)^2\\
& =\sigma_R^2-\frac{1}{2 N} \sum_{i=1}^N (S_i-R_i)^2 \\ 
& =\sigma_R \sigma_S-\frac{1}{2 N} \sum_{i=1}^N (S_i-R_i)^2.
\end{aligned}
\end{equation}
As such, we know
\begin{equation}
    \begin{aligned} 
    r_s
    &=\frac{\text{cov}(S, R)}{\sigma_S \sigma_{R}}\\
    &=\frac{\sigma_R \sigma_S-\frac{1}{2 N} \sum_{i=1}^N (S_i-R_i)^2}{\sigma_S \sigma_{R}}\\
    &=1- \frac{\sum_{i=1}^N (S_i-R_i)^2}{2N(N^2-1)/12}\\
    &=1-\frac{6\sum_{i=1}^N (S_i-R_i)^2}{N(N^2-1)}
    \end{aligned}
    \label{eq:rs}
\end{equation}
With the above equation, we can re-write $\sum_{i=1}^N (S_i-R_i)^2$ as follows:
\begin{equation}
    \sum_{i=1}^N (S_i-R_i)^2=\frac{N(N^2-1)(1-r_s)}{6}.
\end{equation}
As $N(N^2-1)/6$ is a constant when $N$ is fixed, we can re-write loss $\mathcal{L}_o$ in Eq.~(6) as follows:
    \begin{equation}
    \begin{aligned}
\mathcal{L}_o
&=\frac{\sum\limits_{\boldsymbol{\phi}\in{\Phi}_p}(1-r(S,S^\phi))}{\sum\limits_{\boldsymbol{\phi}\in{\Phi}_p}(1-r(S,S^\phi))+\sum\limits_{\boldsymbol{\phi}\in{\Phi}_n}(1-r(S,S^\phi))}\\
&=\frac{1}{1+\frac{\sum\limits_{\boldsymbol{\phi}\in{\Phi}_n}(1-r(S,S^\phi))}{\sum\limits_{\boldsymbol{\phi}\in{\Phi}_p}(1-r(S,S^\phi))}}.
    \end{aligned}
    \label{eq:loss2}
    \end{equation}
According to Eq.~(\ref{eq:rs}), we know $0\leq r_s \leq 1$. Therefore, minimizing the loss in Eq.(\ref{eq:loss2}) equals maximizing $\sum\nolimits_{\boldsymbol{\phi}\in{\Phi}_n}(1-r(S,S^\phi))/\sum\nolimits_{\boldsymbol{\phi}\in{\Phi}_p}(1-r(S,S^\phi))$, which is always a positive number. Thus, we have
    \begin{equation}
       \min \mathcal{L}_o \equiv \max \frac{\sum\limits_{\boldsymbol{\phi}\in{\Phi}_n}(1-r(\hat{\phi},\phi))}{\sum\limits_{\boldsymbol{\phi}\in{\Phi}_p}(1-r(\hat{\phi},\phi))}.
    \end{equation}

%Since $N$ is a constant, minimizing $\sum_{i=1}^N (S_i-R_i)^2$ equals maximizing $r_s$. Considering the contrastive loss in Eq.~(\ref{eq:loss}), we know minimizing it equals maximizing the Spearman's rank correlation coefficient~\cite{spearman1904proof} between learned orders and good orders (i.e., $\phi\in\Phi_p$), while minimizing the coefficient between learned orders and bad orders (i.e., $\phi\in\Phi_n$).
\end{proof}

\section{Implementation}\label{app:implementation}
The experiments are implemented with Pytorch 2.0.1 \cite{paszke2019pytorch} on 4 NVIDIA A100 GPUs each with 80GB memory. Unless otherwise stated, we set the number of experts $K$, orders $M$, and layers to 3, 2, 3, respectively. We obtain other best hyper-parameters via grid search with the range of learning rate from $10^{-1}$ to $10^{-3}$, and weight decay from $10^{-3}$ to $10^{-5}$, with each configuration run for 100 epochs. The batch size is set to be 16, and the model is trained with Adam optimizer. For the implementation of GNN baselines, we adopt the best settings from \textit{graphgym}\footnotemark[1]\footnotetext[1]{https://github.com/snap-stanford/GraphGym}. And we use the official public available code for implementation of GraphGPS\footnotemark[2]\footnotetext[1]{https://github.com/rampasek/GraphGPS} and Graph-Mamba\footnotemark[3]\footnotetext[3]{https://github.com/bowang-lab/Graph-Mamba}.

\section{Baseline Details}
We compare our framework with baselines used by the NeuroGraph benchmark and two state-of-the-art models GraphGPS and Graph-Mamba that extract long-range dependencies within the graph data. GraphGPS employs a modular framework that integrates SE, PE, MPNN, and a graph transformer, where it allows the replacement of fully-connected Transformer attention with its sparse alternatives. Graph-Mamba is the pioneering work to applies state space models (SSMs) for non-sequential graph data, where it captures long-range dependencies with linear time complexity.

\section{Efficiency Study}
In this section, we evaluate the computational efficiency of different models by analyzing the average training time per epoch across various Human Connectome Project (HCP) datasets, using 4 A100 GPUs. The results are summarized in Table 1, where training time is measured in seconds.

\end{document}